\ifdefined\pdfminorversion\pdfminorversion=7\fi
\documentclass[10pt]{article}

\usepackage[letterpaper,margin=1in]{geometry}
\usepackage[T1]{fontenc}
\usepackage[sc]{mathpazo}
\usepackage[protrusion=true,expansion=true]{microtype}
\usepackage{newunicodechar}
\newunicodechar{’}{\textquoteright}
\newunicodechar{–}{\textendash}

\usepackage{amsmath,amsfonts,bm}

\def\1{\bm{1}}

\DeclareMathAlphabet{\mathsfit}{\encodingdefault}{\sfdefault}{m}{sl}
\SetMathAlphabet{\mathsfit}{bold}{\encodingdefault}{\sfdefault}{bx}{n}

\usepackage{xcolor}
\usepackage{stmaryrd}
\usepackage{pifont}
\usepackage{graphicx}
\usepackage{colortbl}
\usepackage{natbib}
\setcitestyle{authoryear,round,citesep={;},aysep={,},yysep={;}}
\usepackage{url}
\usepackage{booktabs}
\usepackage{multirow}
\usepackage{array}
\usepackage{tabularx}
\usepackage{siunitx}
\newcolumntype{N}[1]{S[table-format=#1]}
\newcolumntype{P}{S[table-format=3.1,round-precision=1]}
\usepackage{amsmath,amssymb,amsthm}

\usepackage{adjustbox}
\definecolor{oursbg}{RGB}{234,242,255}
\definecolor{spdgreen}{RGB}{20,125,60}
\newcommand{\ours}{\rowcolor{oursbg}}
\newcommand{\spd}[2]{\textbf{#1}\textsubscript{\,\color{spdgreen}#2$\times$}}
\newcommand{\rat}[1]{\,{\color{gray}\footnotesize(#1$\times$)}}
\newcommand{\tgroup}[2]{\addlinespace[3pt]\multicolumn{#1}{@{}l}{\textit{#2}}\\[1pt]}
\theoremstyle{plain}
\newtheorem{theorem}{Theorem}[section]
\newtheorem{lemma}[theorem]{Lemma}
\newtheorem{proposition}[theorem]{Proposition}
\theoremstyle{definition}

\theoremstyle{plain}

\definecolor{brandblue}{RGB}{22,119,255}
\definecolor{linkblue}{RGB}{0,82,170}
\usepackage[font=small,labelfont=bf,labelsep=period]{caption}
\usepackage{enumitem}
\setlist{itemsep=2pt,topsep=4pt,parsep=1pt}
\usepackage{titlesec}
\titleformat{\section}{\large\bfseries}{\thesection}{0.7em}{}
\titleformat{\subsection}{\normalsize\bfseries}{\thesubsection}{0.7em}{}
\titleformat{\subsubsection}{\normalsize\bfseries\itshape}{\thesubsubsection}{0.7em}{}
\titlespacing*{\section}{0pt}{12pt plus 3pt minus 2pt}{4pt plus 1pt}
\titlespacing*{\subsection}{0pt}{9pt plus 2pt minus 2pt}{3pt plus 1pt}
\titlespacing*{\paragraph}{0pt}{5pt plus 2pt minus 1pt}{0.7em}

\usepackage{hyperref}
\hypersetup{colorlinks=true,linkcolor=linkblue,citecolor=linkblue,urlcolor=linkblue,
  pdftitle={DConstrain: Compiling Context-Parameterized Transitions for Diffusion Language Model Constrained Decoding},
  pdfauthor={Shidong Cao, Jing Ma, Wenhao Chai, Yuchen Tian, Lun Du, Da Zheng}}

\begin{document}
\thispagestyle{plain}

\newcommand{\logodivider}{\hspace{8pt}{\color{gray!55}\rule{0.6pt}{0.30in}}\hspace{8pt}}
\noindent
\includegraphics[height=0.30in]{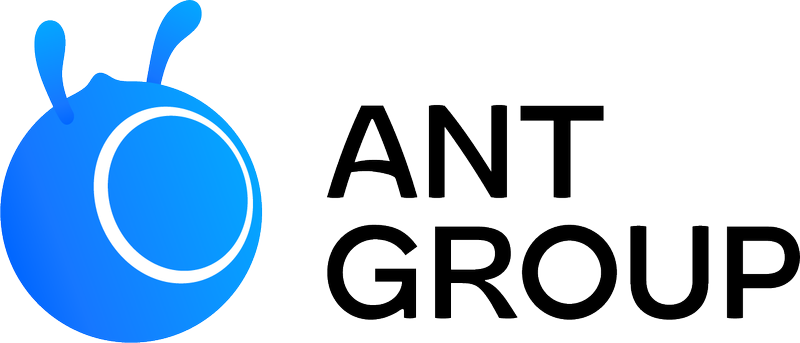}%
\logodivider
\raisebox{0.05in}{\includegraphics[height=0.20in]{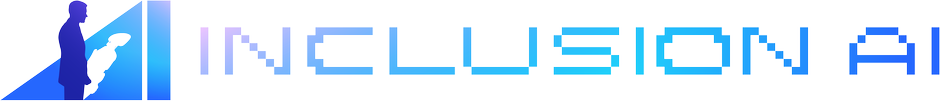}}%
\hfill
\raisebox{-0.03in}{\includegraphics[height=0.36in]{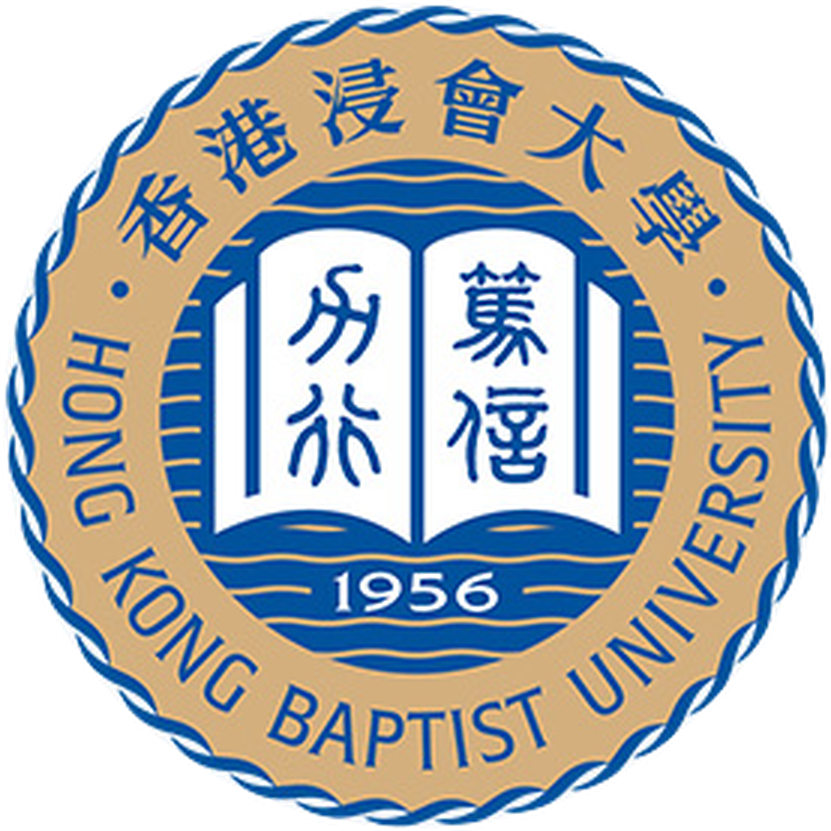}}%
\logodivider
\raisebox{-0.03in}{\includegraphics[height=0.36in]{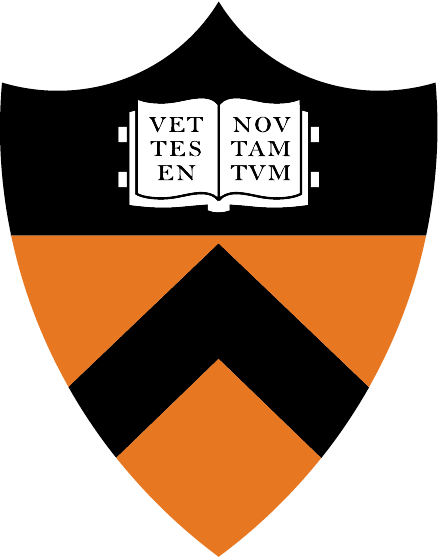}}
\par\vspace{4pt}
{\color{brandblue}\hrule height 0.6pt}
\vspace{16pt}
\begin{center}
{\LARGE\bfseries DConstrain: Compiling Context-Parameterized Transitions\\[3pt]
for Diffusion Language Model Constrained Decoding\par}
\vspace{14pt}
{\color{brandblue}\hrule height 0.6pt}
\vspace{14pt}
{\large\bfseries Shidong~Cao$^{1}$\quad\ Jing~Ma$^{1\dagger}$\quad\ Wenhao~Chai$^{2}$\quad\ Yuchen~Tian$^{1}$\quad\ Lun~Du$^{3\dagger}$\quad\ Da~Zheng$^{3\dagger}$\par}
\vspace{6pt}
{\normalsize $^{1}$Hong Kong Baptist University\qquad $^{2}$Princeton University\qquad $^{3}$Ant Group\par}
\vspace{4pt}
{\small\ttfamily \{cssdcao,majing\}@comp.hkbu.edu.hk\qquad \{dulun.dl,zhengda.zheng\}@antgroup.com\par}
\end{center}
{\let\thefootnote\relax\footnotetext{$^{\dagger}$Corresponding authors.}}
\vspace{6pt}
\begin{center}{\large\bfseries Abstract}\end{center}
\vspace{-4pt}
\begin{list}{}{\leftmargin=0.55in\rightmargin=0.55in}\item\relax
Diffusion language models (DLMs) decode multiple tokens in parallel, so individually valid tokens can be incompatible together under a schema. Exact joint inference over a grammar graph enforces compatibility, but it requires token-labeled successors across many hypothetical parser configurations. Encoding stacks and field histories in compiled states can cause exponential growth, which forces prior exact decoders to fix property orders and exclude valid outputs. We present \textbf{DConstrain}, which compiles schemas into \emph{context-parameterized transition programs}. Cold compilation specializes token consumption to local grammar states while retaining context-dependent stack operations, environment guards, and caller-dependent deferred parsing symbolically. Tracking field histories as runtime values supports arbitrary property order without enumerating field subsets offline. At runtime, warm compilation binds reached configurations to these programs to construct exact successors, overlapping with the model forward pass because graph construction requires no model scores. Furthermore, \emph{token-effect grouping} merges tokens with identical compiled effects before runtime contexts are bound, so groups are reused across stacks and environments even when concrete successors differ. Objective-specific score aggregation and token restoration provably preserve both the Viterbi optimum and the token-labeled path distribution of ungrouped inference. Across three benchmarks and two DLM families, DConstrain compiles 100\% of BFCL schemas under arbitrary property orders, whereas an optimized finite-state baseline compiles only 71.30--85.81\%. Under arbitrary property order, it reduces BFCL end-to-end decoding latency by 6.65--46.26$\times$ with 100\% schema validity and raises DiffusionGemma's BFCL accuracy from 37.27\% under fixed order to 81.56\%.

\end{list}
\vspace{4pt}

\section{Introduction}
\label{sec:introduction}

Diffusion language models (DLMs) generate text by iteratively denoising multiple tokens in parallel~\citep{austin2021d3pm,lou2024sedd,sahoo2024mdlm,nie2025llada,ye2025dream}. Parallel decoding reduces latency and makes DLMs attractive for agentic tasks such as tool calling~\citep{schick2023toolformer,qin2024toolllm}, API interaction, and structured data extraction. These tasks require outputs that strictly follow a schema~\citep{patil2025bfcl,geng2025jsonschemabench}, yet unconstrained DLMs often violate schema rules~\citep{suresh2025dingo,dang2026constrained}. In autoregressive (AR) models, schema compliance is enforced by masking invalid next tokens for a fixed left-to-right prefix~\citep{willard2023guided,llguidance,dong2025xgrammar}. This mechanism does not carry over directly to DLMs. Because a DLM fills several masked positions at once, two tokens that are each valid on their own can be invalid together.

Existing DLM constrained decoders follow two main paradigms. \emph{Verification-based methods} sample masked positions independently and then check whether the proposed output can still be completed~\citep{mundler2026cfg,zhang2026lave,jin2026epic}. Because independent proposals ignore dependencies among positions, decoding many tokens at once causes frequent conflicts and requires repeated resampling or heuristic fallbacks under finite budgets~\citep{mundler2026cfg,jin2026epic}. Conversely, \emph{automata-based methods} enforce compatibility through joint inference~\citep{suresh2025dingo,dang2026constrained,thimmaiah2026twister}. For the active denoising window (the \emph{canvas}), they propagate the grammar position by position into a \emph{canvas-specific grammar graph} whose paths represent sequences of mutually compatible tokens. The decoder then identifies the highest-scoring path or samples according to the graph's path weights. However, these inference algorithms assume a finite automaton with readily accessible token transitions, leaving the cost of constructing this representation for complex schemas largely unaddressed.

This bottleneck arises because DLM joint inference requires a richer \emph{successor-construction interface} than AR decoding. An AR decoder only needs a boolean mask of valid next tokens for a single realized prefix~\citep{dong2025xgrammar}. In contrast, joint inference must compute the exact successor \emph{parser configurations} for candidate tokens across many hypothetical configurations before any token is committed. Each configuration consists of a \emph{local grammar state} (a position inside one grammar rule) and a \emph{runtime context} (a stack of return points and a parsing environment, such as the set of fields already written). Finite-state methods encode runtime contexts as separate automaton states, and DINGO additionally precomputes transitions across the full vocabulary~\citep{suresh2025dingo}. In real schemas, nested objects multiply possible stack contexts, while unordered properties introduce exponentially many field histories. Even context-free grammars (CFGs) can require exponential size to represent all permutations of distinct fields~\citep{ellul2005regular}. To keep compiled automata tractable, prior exact decoders restrict schemas to a fixed property order~\citep{suresh2025dingo,dang2026constrained,thimmaiah2026twister}, excluding valid JSON outputs that the model may prefer.

Our key observation is that \emph{context-specific successors need not be compiled as context-specific transitions}. Consider a token containing two closing braces, \texttt{\}\}}, in a nested JSON object. After the first brace closes the inner object, parsing the remaining brace requires returning to the caller. If the parent still lacks a required field, the token is rejected; otherwise, it may close the parent. The return point depends on the stack, and field completion depends on the parsing environment. These contexts can share one compiled procedure: check field completion, return to the caller, and resume parsing the remaining bytes. Rather than precomputing successors for every combination of stack and field history, we compile how to compute them with the runtime context left as a parameter. This separation allows exact successors to be constructed at runtime without enumerating all possible contexts during compilation.

Building on this insight, we propose \textbf{DConstrain}, which compiles schemas into \emph{context-parameterized transition programs} (Figure~\ref{fig:method_overview}). \emph{Cold compilation} specializes token consumption to local grammar states: it precomputes byte-level parsing and local transitions wherever inputs are known, and retains context-dependent stack operations, environment guards, and deferred parsing symbolically. Tracking field histories as runtime values supports arbitrary property order without enumerating field subsets offline. At runtime, \emph{warm compilation} instantiates these programs for configurations reachable in the current canvas, resolves context-dependent operations, and constructs the exact successor relations required for joint inference. Since warm graph construction does not require model scores, it overlaps with the model forward pass. Finally, \emph{token-effect grouping} identifies equivalent compiled token effects during cold compilation, before runtime contexts are bound. The resulting groups can be reused across different stacks and environments, sharing successor computation even when the concrete successors differ across contexts. Objective-specific score aggregation and token restoration preserve the exact Viterbi optimum and token-labeled path distribution.

\begin{samepage}
Our contributions are summarized as follows:
\begin{itemize}
\item \textbf{Successor-Construction Bottleneck.} We identify token-labeled successor construction as the primary scalability bottleneck in automata-based DLM decoding. Materializing context-specific transitions incurs combinatorial growth with nesting and field histories. Even CFGs can require exponential size for arbitrary permutations of distinct fields, motivating fixed property orders in prior exact decoders.

\item \textbf{Context-Parameterized Compilation.} We introduce context-parameterized transition programs paired with a two-stage compiler and compile-time token-effect grouping. We prove that program execution preserves the reference successor relation and that grouping preserves both the Viterbi optimum and the full token-labeled path distribution of ungrouped inference. The compiled programs can be reused by any constrained-decoding algorithm that consumes token-labeled grammar successors.

\item \textbf{Exact Unordered Decoding.} Across three structured-output benchmarks and multiple DLM tokenizers, DConstrain substantially improves compilation coverage and latency under arbitrary property order. On BFCL, exact unordered constraints also improve task accuracy over both fixed-order constraints and unconstrained decoding.
\end{itemize}
\end{samepage}

\section{Related Work}
\label{sec:related_work}

\paragraph{Autoregressive Constrained Decoding.}
Grammar-constrained decoding is widely used to guarantee structured outputs in autoregressive LLMs. At each step, the decoder masks invalid tokens for the realized prefix and advances the parser with the selected token. Early methods introduced incremental parsing and semantic constraints for structured NLP tasks~\citep{scholak2021picard,poesia2022synchromesh,geng2023gcd}. Subsequent work corrected the distributional distortion of token-level masking through adaptive sampling~\citep{park2024grammaraligned}. Later systems focused on efficient tokenizer-aware mask generation~\citep{willard2023guided,beurerkellner2024domino,koo2024automata,llguidance,dong2025xgrammar}, and XGrammar 2 added JIT compilation with cross-grammar cache reuse~\citep{li2026xgrammar2}. Our compiler design is closest to two of these systems. XGrammar~\citep{dong2025xgrammar} precomputes stack-independent valid tokens at each grammar position and checks only the remaining tokens at runtime. \citet{koo2024automata} compose a detokenizing transducer with a pushdown automaton so that each token transition carries explicit stack operations. DConstrain also compiles context-independent work once, but it targets a different \emph{successor-construction interface}. A validity mask only identifies admissible next tokens for a single prefix. In contrast, DLM joint inference requires the full transition map $(c,v)\mapsto \mathsf{T}_{\mathcal G}(c,v)$ across many hypothetical configurations before any token is committed. DConstrain further keeps the written-field history as a runtime parameter, which no compact grammar can encode (Proposition~\ref{prop:unordered_properties}).

\paragraph{DLM Constrained Decoding.}
Early diffusion language models supported soft control by steering continuous latents with classifier gradients~\citep{li2022diffusionlm}, but soft guidance cannot guarantee valid syntax. Decoders for hard constraints fall into two families. The first family consists of \emph{verification-based decoders}, which sample candidate tokens independently and resample or recover when a grammar check fails. Constrained Diffusion~\citep{mundler2026cfg} checks whether the partial output can still be completed. LAVE~\citep{zhang2026lave} validates each proposed token against sampled lookahead completions. EPIC~\citep{jin2026epic} accelerates completion checks with memoized lexing and Earley-style parsing~\citep{earley1970} to commit verified token batches in parallel. Although these decoders avoid finite-state materialization, a feasibility check only ensures that some valid continuation exists. At each step, it recovers neither the highest-scoring joint assignment nor an exact sample from the constrained distribution. In our BFCL experiments, LAVE-Opt and EPIC-Opt have lower completion rates and throughput than DConstrain (Tables~\ref{tab:runtime_comparison} and~\ref{tab:runtime_common_success}).

The second family consists of \emph{automata-based methods}, which couple token choices through joint inference over finite automata~\citep{suresh2025dingo,dang2026constrained,thimmaiah2026twister}. DINGO~\citep{suresh2025dingo} builds a DFA and precomputes vocabulary-level token transitions. \citet{dang2026constrained} perform chain and tree inference over a finite automaton whose edges group tokens and sum their probabilities. Building on such step-exact samplers, \citet{thimmaiah2026twister} use twisted sequential Monte Carlo to correct trajectory bias across denoising steps. However, all three methods compile JSON schemas under a fixed property order to keep the automaton finite and small. Interpreting a pushdown automaton directly at runtime would avoid materializing context-specific states, but it would repeatedly execute token-consumption logic for every reached configuration. DConstrain resolves this tension by compiling reusable local grammar--token programs while leaving the runtime context unbound. Within this framework, our token-effect grouping generalizes the edge-level aggregation of \citet{dang2026constrained} to parameterized transitions. As a result, a single token partition per local grammar state serves every runtime context.

\section{Preliminaries}
\label{sec:preliminaries}

A JSON schema $\mathcal{G}$~\citep{jsonschema2022core} defines a grammar whose rules can nest. For example, an object property can itself hold an object. We therefore parse $\mathcal{G}$ with a pushdown parser~\citep{hopcroft2007automata}, where each grammar rule is represented as a small automaton and entering a nested rule pushes a return point onto a stack. At any point, the parser state is captured by a configuration $c=(q,s,e)$. Here $q$ is the \emph{local grammar state}, which marks a node inside the current rule's automaton. The stack $s$ holds pending return points, and the parsing environment $e$ records earlier parsing choices, such as which object fields have already been written. Together, $(s,e)$ form the \emph{runtime context} of $c$. Given a configuration $c$ and a vocabulary token $v\in\mathcal{V}$, the token-transition relation $\mathsf{T}_{\mathcal{G}}(c,v)$ returns the set of valid successor configurations $c'$ after consuming $v$. This set omits statically certified dead configurations that can never lead to a complete valid output (see Appendix~\ref{app:grammar_background} for formal parser semantics).

\paragraph{Automata-Based DLM Constrained Decoding.}
Automata-based methods apply $\mathsf{T}_{\mathcal{G}}$ across an active decoding window to coordinate token choices at multiple unresolved positions~\citep{suresh2025dingo,dang2026constrained}. Consider a denoising iteration $t$ and its canvas of $H$ token positions. Let $\mathcal{V}_i^t\subseteq\mathcal{V}$ denote the candidate tokens at position $i$; a position that is already fixed has a singleton candidate set. Starting from the parser configurations $\mathcal{C}_0$ compatible with the fixed prefix $\boldsymbol{u}$, the token-labeled transitions and reachable configurations at each position are defined recursively as
\begin{equation}
\begin{aligned}
E_i
&=
\bigl\{(c,v,c') \mid
c\in\mathcal{C}_{i-1},\;
v\in\mathcal{V}_i^t,\;
c'\in\mathsf{T}_{\mathcal{G}}(c,v)\bigr\}, \\
\mathcal{C}_i
&=
\bigl\{c' \mid \exists c,v:\ (c,v,c')\in E_i\bigr\},
\qquad i=1,\ldots,H.
\end{aligned}
\label{eq:canvas_grammar_graph}
\end{equation}
The position-indexed layers $\{\mathcal{C}_i\}_{i=0}^{H}$ and transitions $\{E_i\}_{i=1}^{H}$ form a layered \emph{canvas-specific grammar graph} $\mathcal{F}_t$. A full-window path through $\mathcal{F}_t$ is admissible if its terminal configuration satisfies the required acceptance or continuation condition.

Each denoising iteration constructs and resolves $\mathcal{F}_t$ in two stages. First, \emph{grammar propagation} builds the graph layers from left to right: starting from $\mathcal{C}_0$, it advances every reached configuration with every candidate token to collect its successors. Notably, propagation depends only on parser configurations and candidate sets, not on model scores. Second, \emph{joint inference} weights each edge $(c,v,c')\in E_i$ using the token probability $p_i(v)$ from the DLM predictions $P_t$ (defined in Appendix~\ref{app:dlm_background}). Max-product dynamic programming then extracts a highest-weight admissible path via the Viterbi algorithm~\citep{viterbi1967}, while sum-product computes path-weight sums for exact probabilistic sampling~\citep{goodman1999semiring,dang2026constrained}.

\paragraph{Problem Formulation.}
\label{sec:problem_formulation}
All automata-based decoders share this two-stage pipeline, and their efficiency hinges on how grammar propagation answers each successor query $\mathsf{T}_{\mathcal{G}}(c,v)$. In AR constrained decoding, each forward pass predicts a single position for a fixed prefix. The decoder only needs a boolean mask of valid next tokens, after which it samples one token $v$ and advances the parser once~\citep{dong2025xgrammar}. In contrast, a DLM leaves many positions unresolved at once, and each position can be reached by many hypothetical parser configurations. Before committing any token, grammar propagation must return exact successor sets rather than boolean masks, and it must evaluate every candidate token at every reached configuration. Without query reuse or token grouping, exhaustive propagation requires
\begin{equation}
N_{\mathrm{query}}^{\mathrm{naive}}
=
\sum_{i=1}^{H}
|\mathcal{C}_{i-1}|\,|\mathcal{V}_i^t|
\label{eq:transition_query_count}
\end{equation}
configuration--token queries per canvas.

To avoid expensive online parsing during propagation, finite-state pipelines such as DINGO~\citep{suresh2025dingo} precompute token-level transitions offline during \emph{cold compilation}. Given a materialized finite state space $Q_{\mathcal{G}}$, a dense transition table stores an entry for each of the $|Q_{\mathcal{G}}||\mathcal{V}|$ state--token pairs, turning each online query into a table lookup. However, flattening runtime contexts into explicit finite states duplicates shared grammar logic and precomputes transitions for contexts that a canvas may never reach. The following propositions quantify this state explosion.

\begin{proposition}[Context expansion; Appendix~\ref{sec:complexity_analysis}]
\label{prop:context_expansion}
Let $Q_{\mathrm{loc}}$ be the set of local grammar states, $g>1$ the number of continuation-frame symbols, $d_{\max}$ the stack-depth bound, and $B$ the maximum number of environment values per local-state--stack pair. A context-expanded state space contains
\begin{equation}
S_{d_{\max}}\;\le\;|Q_{\mathrm{loc}}|\,B\sum_{\ell=0}^{d_{\max}}g^{\ell}\;=\;O\!\left(|Q_{\mathrm{loc}}|\,B\,g^{d_{\max}}\right)
\label{eq:main_context_expansion}
\end{equation}
states. In particular, for bounded-depth brackets of $g$ types, a minimal complete DFA requires $1+\sum_{\ell=0}^{d_{\max}}g^{\ell}=\Theta(g^{d_{\max}})$ states. A dense context-expanded transition table therefore stores $S_{d_{\max}}|\mathcal{V}|$ entries, whereas a factorized compiler stores at most $|Q_{\mathrm{loc}}||\mathcal{V}|$ local-state--token entries.
\end{proposition}

\emph{Proof sketch.} Each context-expanded state pairs a local state with a stack of depth at most $d_{\max}$ over $g$ frame symbols and one environment value; counting these triples yields Eq.~\eqref{eq:main_context_expansion}. For bounded-depth brackets, any two distinct stacks of open brackets require distinct closing sequences. By the Myhill--Nerode theorem~\citep{hopcroft2007automata}, a minimal complete DFA must assign a separate state to every valid stack plus one sink state, which bounds the number of table entries required.

Proposition~\ref{prop:context_expansion} leaves the environment factor $B$ abstract. When a schema includes unordered object properties, $B$ itself grows exponentially, and replacing finite automata with context-free grammars does not avoid this blow-up.

\begin{proposition}[Unordered properties; Appendix~\ref{sec:complexity_analysis}]
\label{prop:unordered_properties}
For every $k\ge1$, there exists an object schema with $k$ required unordered properties such that every DFA accepting its valid serializations has at least $2^k$ states, and every context-free grammar generating them has size $2^{\Omega(k)}$. A factorized compiler instead requires only $O(k)$ local grammar states for this object and tracks its field history in a $k$-bit runtime environment value.
\end{proposition}

\emph{Proof sketch.} Any two prefixes that have written different subsets of required fields admit different sets of valid completions, so the Myhill--Nerode theorem separates all $2^k$ subsets. For the CFG bound, intersecting the schema language with a regular language that fixes field values and counts, and then projecting onto property names, yields the language of all permutations of $k$ distinct symbols. Every CFG generating this permutation language has size $2^{\Omega(k)}$~\citep{ellul2005regular}. Because regular intersection and homomorphic projection increase CFG size by at most a polynomial factor, the original grammar must also have size $2^{\Omega(k)}$.

Propositions~\ref{prop:context_expansion} and~\ref{prop:unordered_properties} bound the size of static finite-state and CFG representations rather than the runtime cost of parsing a single string. For offline finite-state compilers, however, representation size directly dictates compilation cost, which explains the sharp growth with nesting depth and property count observed in Section~\ref{sec:analysis}.

To overcome this bottleneck, we seek a compiled artifact $\mathcal{A}_{\mathcal{G}}$ containing executable successor-construction procedures, with the runtime context supplied at query time. These procedures must implement the \emph{exact successor-construction interface}
\begin{equation}
\operatorname{Succ}(\mathcal{A}_{\mathcal{G}},c,v)
=
\mathsf{T}_{\mathcal{G}}(c,v)
\label{eq:successor_contract}
\end{equation}
for every configuration--token pair queried during grammar propagation. Our goal is to minimize cold-compilation cost, compiled storage, and online query latency without materializing the exponentially large context-expanded state space.

\section{DConstrain}
\label{sec:methodology}

DConstrain compiles schemas into \emph{context-parameterized transition programs} (Figure~\ref{fig:method_overview}). These programs implement the successor interface in Eq.~\eqref{eq:successor_contract}, independently of the downstream constrained-decoding algorithm. Cold compilation specializes token consumption to local grammar states while leaving the runtime context unbound and retaining context-dependent operations symbolically (Section~\ref{sec:pushdown_factorization}). Warm compilation executes these programs for reached configurations to construct exact successors (Section~\ref{sec:on_demand_instantiation}). Token-effect grouping partitions tokens by their canonical compiled effects during cold compilation and reuses these partitions during warm compilation (Section~\ref{sec:token_effect_reduction}). Appendices~\ref{sec:factorized_correctness} and~\ref{sec:complexity_analysis} provide correctness and complexity analyses.

\begin{figure*}[t]
\centering
\includegraphics[width=\textwidth]{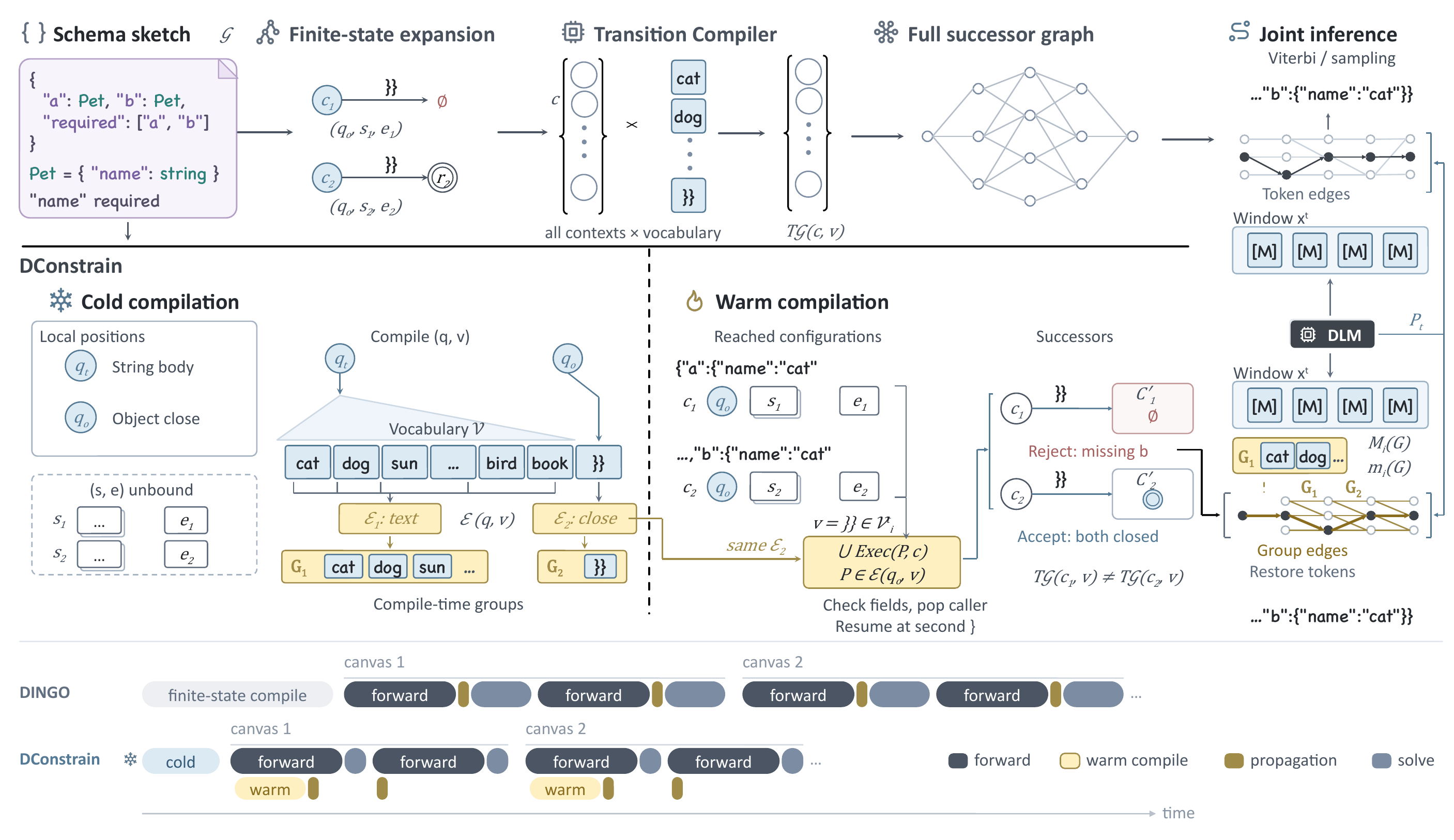}
\caption{Overview of DConstrain on a nested-object schema with required properties \texttt{a} and \texttt{b} sharing the same \texttt{Pet} rule. \textbf{Top:} A finite-state pipeline such as DINGO encodes the grammar position and runtime context together in each state and precomputes transitions across the full vocabulary. States $c_1$ and $c_2$ share the local position $q_o$ but carry their own stacks and field histories. The token \texttt{\}\}} is rejected from $c_1$, whose parent still requires \texttt{b}, but closes both objects from $c_2$. \textbf{Bottom:} Cold compilation takes a local grammar position $q$ and token $v$, leaving $(s,e)$ unbound, and produces context-parameterized program sets. The local-position legend identifies $q_t$ as the string body and $q_o$ as the child-object close. Tokens with identical canonical effects at each position form compile-time groups. For \texttt{\}\}}, the shared program set $\mathcal{E}_2$ defers parsing the second brace until the caller is known. Warm compilation binds the concrete stack and environment, returning an empty successor set for $c_1$ and an accepting configuration for $c_2$. \textbf{Right:} Joint inference uses group-level edges with sum or max score aggregation before restoring concrete tokens. \textbf{Timeline:} Cold compilation runs once offline. Warm compilation overlaps with the first forward pass of each canvas, and grammar propagation overlaps with every forward pass before the inference solve. DINGO runs propagation and solve sequentially after each forward pass. Graphs and block lengths are schematic.}
\label{fig:method_overview}
\end{figure*}

\subsection{Context-Parameterized Transition Compilation}
\label{sec:pushdown_factorization}

Given a target schema $\mathcal{G}$, DConstrain represents local grammar rules as a recursive transition network (RTN)~\citep{woods1970rtn}, keeping the call stack $s$ and parsing environment $e$ outside the RTN as runtime context. Specifically, we decompose the environment as $e=(d,\eta)$, where $d$ tracks the active container depth and $\eta$ tracks the choice environment (such as the written-field history), giving the full configuration $c=(q,s,d,\eta)$. Cold compilation specializes token consumption to each local grammar state and token pair $(q,v)$ while leaving $(s,d,\eta)$ unbound:
\begin{equation}
(q,v)
\xrightarrow{\operatorname{ColdCompile}}
\mathcal{E}(q,v)
=
\{P_k\}_{k=1}^{K_{q,v}}.
\label{eq:parametric_program_set}
\end{equation}
Each program $P_k=(I_{k,1};\ldots;I_{k,\ell_k})$ is an ordered symbolic execution alternative for consuming $v$ from $q$. Its guards may reject a concrete runtime context, and its continuation may depend on a caller selected by the incoming stack. Cold compilation evaluates byte matching and local grammar operations wherever their inputs are known, retaining context-dependent operations and the input and execution state needed to resume deferred parsing. The resulting programs take two forms: residual programs with a deferred caller-dependent continuation, and compact summaries with a statically determined target state.

When token consumption or token-boundary processing must pop a return frame from the incoming stack, the caller needed to continue parsing is unknown during cold compilation. The compiler resolves pushes and pops whose return destinations can be determined locally. At the first return that requires the incoming stack, it records a \emph{residual program} with the remaining token input and the current symbolic execution state. Warm compilation uses the concrete stack to resume the caller-dependent continuation. A shared program can therefore compute exact successors even when its final local state cannot be determined offline.

When a program does not require a caller from the incoming stack $s$, symbolic execution reduces it to a compact summary $P=(q',\rho,\Delta,m,\Omega)$. Here $q'$ is the resulting local grammar state, and $\rho$ is the sequence of new return frames pushed (and not yet popped) while consuming $v$. The integers $\Delta$ and $m$ record the net change and minimum intermediate change in container depth $d$. Finally, $\Omega$ is an ordered sequence of environment guards and updates on $\eta$, such as verifying that field \texttt{a} has not yet appeared and recording its presence. During warm compilation, binding a summary $P$ to a concrete runtime context $(s,d,\eta)$ evaluates directly to
\begin{equation}
\operatorname{Bind}(P;s,d,\eta)=
\begin{cases}
\{(q',\,s\mathbin{\|}\rho,\,d+\Delta,\,\Omega(\eta))\}, & d+m\ge 0,\ \Omega(\eta)\ne\bot,\\
\varnothing, & \text{otherwise},
\end{cases}
\label{eq:main_summary_binding}
\end{equation}
where $\mathbin{\|}$ appends frames to the stack and $\bot$ indicates a violated guard (see Appendix~\ref{sec:factorized_correctness} for the full specification with auxiliary metadata). For these summaries, binding evaluates the remaining guards and updates under $(s,d,\eta)$ and constructs the successor without recompiling grammar logic for that context. Together, summaries and residuals allow configurations with the same local state $q$ to reuse $\mathcal{E}(q,v)$ even when their successors differ.

The choice environment is likewise handled through guarded runtime operations. For an object with $k$ fields under arbitrary property order, $\eta$ stores the set of written fields as a runtime bitmask, while the compiled operator $\Omega$ encodes the guards for field uniqueness and required-field completion. A single program set $\mathcal{E}(q,v)$ is shared across the $2^k$ possible field histories, so cold compilation does not enumerate those subsets as local grammar states. Concrete histories remain runtime values, and program execution preserves the order and scope of their guards and updates.

\subsection{Warm Compilation}
\label{sec:on_demand_instantiation}

Warm compilation executes the remaining operations of a compiled program under a concrete runtime context. We first specify its token-level successor interface: for a reached configuration $c=(q,s,e)\in\mathcal{C}_{i-1}$ and a candidate token $v\in\mathcal{V}_i^t$ at position $i$, it returns
\begin{equation}
\widehat{\mathsf{T}}_{\mathcal{G}}(c,v)
=
\bigcup_{P\in\mathcal{E}(q,v)}
\operatorname{Exec}(P,c).
\label{eq:warm_instantiation}
\end{equation}
Here $\operatorname{Exec}(P,c)$ applies Eq.~\eqref{eq:main_summary_binding} to ordinary summaries, resumes residual programs under the concrete caller, and completes token-boundary processing. Failed guards contribute no successor. The resulting context-specific successors drive the layer-by-layer propagation in Eq.~\eqref{eq:canvas_grammar_graph}. The implementation evaluates this interface through the compile-time partitions introduced in Section~\ref{sec:token_effect_reduction}, using one representative token per reached configuration and active group. Transitions for unreachable contexts are never instantiated.

The following theorem establishes that the two-stage compiler and program execution preserve the reference successor relation, including caller-dependent continuations.

\begin{theorem}[Exact successor construction; Appendix~\ref{sec:factorized_correctness}]
\label{thm:main_exactness}
For every reached configuration $c$ and candidate token $v$, warm compilation returns
\begin{equation}
\widehat{\mathsf{T}}_{\mathcal{G}}(c,v)=\mathsf{T}_{\mathcal{G}}(c,v)
\label{eq:main_exactness}
\end{equation}
under the reference parser semantics. Consequently, running grammar propagation with $\widehat{\mathsf{T}}_{\mathcal{G}}$ reproduces the exact layers $\{\mathcal{C}_i\}$ and transitions $\{E_i\}$ of Eq.~\eqref{eq:canvas_grammar_graph} from the same initial set $\mathcal{C}_0$ and candidate sets $\{\mathcal{V}_i^t\}$. Under the same acceptance or continuation condition, both graphs admit the exact same token-labeled window language.
\end{theorem}

\subsection{Token-Effect Grouping}
\label{sec:token_effect_reduction}

Parameterized programs share transition logic across runtime contexts, but evaluating candidate tokens separately at each reached configuration retains the vocabulary factor in Eq.~\eqref{eq:transition_query_count}. We extend token aggregation in finite-state inference~\citep{dang2026constrained} to context-parameterized programs: during cold compilation, tokens are grouped by their canonical compiled effects at each local state, before any concrete stack or environment is bound. For an unconstrained string rule, for example, many text tokens compile to the same continuation program (Figure~\ref{fig:method_overview}). We define the grouping relation at local state $q$ by
\begin{equation}
v_1\sim_q v_2
\quad\Longleftrightarrow\quad
\operatorname{Canon}\!\left(\mathcal{E}(q,v_1)\right)
=
\operatorname{Canon}\!\left(\mathcal{E}(q,v_2)\right).
\label{eq:token_effect_equivalence}
\end{equation}
Cold compilation records the partition $\Pi_q=\mathcal{V}/{\sim_q}$. Canonicalization retains instruction order, operands, stack and environment operations, and execution-relevant boundary and residual metadata. Canonical equality is a sufficient, not necessary, condition for identical successor sets under every admissible runtime context (Appendix~\ref{sec:factorized_correctness}). The same partition can therefore be reused across different stacks and environments, even though their concrete successor sets may differ.

The offline partition $\Pi_q$ is shared by all configurations with local state $q$. At position $i$, the candidate set selects its active groups, $\Pi_{q,i}^t=\{G\in\Pi_q: G\cap\mathcal{V}_i^t\neq\varnothing\}$. For each reached configuration $c=(q,s,e)$ and active group $G\in\Pi_{q,i}^t$, warm compilation evaluates the successor relation once using any representative token $v_G\in G\cap\mathcal{V}_i^t$:
\begin{equation}
\widehat{\mathsf{T}}_{\mathcal{G}}(c,G)
=
\widehat{\mathsf{T}}_{\mathcal{G}}(c,v_G),
\qquad
v_G\in G\cap\mathcal{V}_i^t.
\label{eq:grouped_transition_instantiation}
\end{equation}
Using these representative successor sets, warm compilation directly constructs one group-level edge $(c,G,c')$ for each returned successor $c'\in\widehat{\mathsf{T}}_{\mathcal{G}}(c,G)$, without first materializing the corresponding token-level edges. This yields the grouped transition layer
\begin{equation}
E_i^{\Pi}=\bigl\{(c,G,c')\;\big|\; c\in\mathcal{C}_{i-1},\ G\in\Pi_{q(c),i}^t,\ c'\in\widehat{\mathsf{T}}_{\mathcal{G}}(c,G)\bigr\}.
\label{eq:grouped_layer}
\end{equation}
As analyzed in Appendix~\ref{sec:complexity_analysis}, the number of successor evaluations now scales with active groups rather than individual tokens:
\begin{equation}
N_{\mathrm{inst}}=\sum_{i=1}^{H}\sum_{c\in\mathcal{C}_{i-1}}\bigl|\Pi_{q(c),i}^t\bigr|\;\le\;\min\Bigl\{N_{\mathrm{query}}^{\mathrm{naive}},\;R_t\max_q|\Pi_q|\Bigr\},
\qquad R_t=\sum_{i=1}^{H}|\mathcal{C}_{i-1}|.
\label{eq:main_instantiation_bound}
\end{equation}
Grouping reduces successor evaluations and the number of edges traversed during joint inference. The bound above concerns successor evaluations; objective-specific score aggregation still uses the individual probabilities of active group members, whose identities remain available for token restoration.

The compile-time partition is independent of model scores and of the inference objective. Joint inference assigns weights to its group-level edges using the current predictions $P_t$. For nonnegative token weights $p_i(v)$, define the objective-specific aggregates
\begin{equation}
M_i(G)=\sum_{v\in G\cap\mathcal{V}_i^t}p_i(v),
\qquad
m_i(G)=\max_{v\in G\cap\mathcal{V}_i^t}p_i(v).
\label{eq:grouped_inference_weights}
\end{equation}
Sum-product inference weights each group-level edge by $M_i(G)$, whereas max-product inference uses $m_i(G)$. For any fixed source configuration $c$, the active groups are disjoint and all tokens within each group have the same successor set. These objective-specific reductions therefore preserve the weighted transition operator between $c$ and every successor configuration $c'$ for $\oplus\in\{+,\max\}$:
\begin{equation}
A_i^{\oplus}(c,c')\;=\;\bigoplus_{v\in\mathcal{V}_i^t:\,c'\in\mathsf{T}_{\mathcal{G}}(c,v)}p_i(v)\;=\;\bigoplus_{G\in\Pi_{q(c),i}^t:\,c'\in\widehat{\mathsf{T}}_{\mathcal{G}}(c,G)}r_i^{\oplus}(G),
\qquad r_i^{+}=M_i,\quad r_i^{\max}=m_i .
\label{eq:main_grouped_operator}
\end{equation}
After selecting a group-level edge, DConstrain restores a token $v\in G\cap\mathcal{V}_i^t$ according to the same inference objective. Sum-product restoration is applied to selected groups with positive mass $M_i(G)>0$. The restoration rules are
\begin{equation}
\Pr(v\mid G)=\frac{p_i(v)}{M_i(G)}\ \ \text{(sum-product)},
\qquad
v\in\operatorname*{arg\,max}_{v'\in G\cap\mathcal{V}_i^t}p_i(v')\ \ \text{(max-product)}.
\label{eq:main_restoration}
\end{equation}

The following theorem shows that grouping preserves weighted inference.

\begin{theorem}[Grouping preserves inference; Appendix~\ref{sec:weighted_grouping}]
\label{thm:main_grouping}
Let $\Pi_q$ be the token-effect partitions of Eq.~\eqref{eq:token_effect_equivalence}, and assume a positive partition function and exact arithmetic. Under both chain and tree inference, grouped sum-product inference with the restoration rule in Eq.~\eqref{eq:main_restoration} preserves both the partition function and the exact token-labeled path distribution of ungrouped inference. Furthermore, grouped max-product inference attains the exact Viterbi optimum and recovers every optimal token sequence when all ties are retained.
\end{theorem}

\paragraph{Solver-agnostic inference interface.}
Because DConstrain preserves the exact successor relation $\mathsf{T}_{\mathcal{G}}(c,v)$ and both weighted transition operators, the same grouped graph representation supports different automata-based solvers. Appendix~\ref{app:tree_sampling} evaluates both Viterbi decoding and the tree-based sampling algorithm of \citet{dang2026constrained}.

\paragraph{Runtime integration.}
Because neither warm compilation nor grammar propagation depends on model logits, CPU workers execute both stages concurrently while the GPU computes the DLM forward pass. Warm compilation runs only during the first forward pass of each canvas. Subsequent denoising iterations on the same canvas only shrink the candidate sets $\mathcal{V}_i^t$ as positions are resolved, so their required successor queries are a subset of those already computed. In every iteration, grammar propagation builds $\mathcal{F}_t$ from these cached successors during the forward pass, and joint inference begins as soon as both $\mathcal{F}_t$ and $P_t$ are ready. Compact compressed sparse row (CSR) layouts and specialized GPU kernels further accelerate the joint inference solve. Appendix~\ref{app:runtime_integration} details the full compiler--runtime pipeline and the optimizations enabled in each experimental setting.

\section{Experiments}
\label{sec:experiments}
\paragraph{Workloads and models.}
We evaluate on BFCL~\citep{patil2025bfcl}, JSONSchemaBench~\citep{geng2025jsonschemabench}, and ComplexFuncBench~\citep{zhong2025complexfuncbench}. BFCL and ComplexFuncBench provide tool declarations for both compilation and task-level evaluation, while JSONSchemaBench stress-tests compilation with diverse and complex schemas. For end-to-end evaluation, we select two recent DLM families with different decoding granularities: LLaDA2.2~\citep{bie2026llada22} with a block size of 32 and DiffusionGemma~\citep{diffusiongemma2026} with a block size of 256. Both decode a long output block by block, as in block diffusion~\citep{arriola2025blockdiffusion}. To reflect realistic serving conditions, we run LLaDA2.2 in SGLang~\citep{zheng2024sglang} and DiffusionGemma in vLLM~\citep{kwon2023vllm} rather than in standalone research inference loops. Cold-compilation experiments need only the corresponding tokenizers.

\paragraph{Comparisons.}
For cold compilation, we compare with DINGO's DFA compiler~\citep{suresh2025dingo}, an NFA adaptation, and DINGO-Opt, our engineering-optimized version of DINGO, under matched schemas and tokenizers. The main BFCL runtime tables compare unconstrained Native-Opt decoding, the verification-based methods LAVE-Opt~\citep{zhang2026lave} and EPIC-Opt~\citep{jin2026epic}, DINGO-Opt, and DConstrain. Table~\ref{tab:runtime_latency_matched} compares DConstrain and DINGO-Opt on their common successful cases, including Any-order. The -Opt suffix identifies baseline implementations with our optimizations and adaptations to the decoding setup. Appendix~\ref{app:additional_results} adds alternative solvers and further baselines.

\paragraph{Cold-compilation targets.}
Table~\ref{tab:cold_compilation} compares compilers that prepare token-labeled successor construction for joint solving. Verification-based methods prepare a different object: a parser and auxiliary structures for deciding whether a proposed partial output admits a legal completion~\citep{mundler2026cfg,zhang2026lave,jin2026epic}. A Boolean verifier can return after finding one witness. A joint solver needs the compatible successor alternatives so that max- or sum-product inference can combine their model weights. DConstrain compiles the programs that construct these alternatives, leaving the inference algorithm downstream. Verifier preparation and successor compilation thus have different output contracts; Tables~\ref{tab:runtime_comparison} and~\ref{tab:runtime_common_success} compare their resulting decoding pipelines in completion, accuracy, and throughput.



\newcommand{\NA}{\multicolumn{1}{c}{\textcolor{gray!65}{\textemdash}}}
\begin{table*}[t]
\centering
\setlength{\abovecaptionskip}{0pt}
\setlength{\belowcaptionskip}{8pt}
\caption{Cold-compilation performance on BFCL, JSONSchemaBench (JSB), and ComplexFuncBench (CFB). Success@180s is the fraction of all attempted schemas compiled successfully within 180\,s; P50/P95 use each method's successful cases. DINGO-Opt$^{\ast}$ retains DINGO's finite-state strategy with the same engineering optimizations as DConstrain (Appendix~\ref{app:runtime_integration}). Speedup$^{\dagger}$ is the row method's mean latency divided by DConstrain's on paired successful cases; timeouts count only toward Success@180s.}
\label{tab:cold_compilation}
\small
\setlength{\tabcolsep}{6pt}
\renewcommand{\arraystretch}{1.10}
\begin{adjustbox}{max width=\textwidth}
\begin{tabular}{@{}lPN{3.2}N{3.2}c@{\hspace{14pt}}PN{3.2}N{3.2}c@{}}
\toprule
& \multicolumn{4}{c}{\textbf{LLaDA2.2}} & \multicolumn{4}{c}{\textbf{DiffusionGemma}} \\
\cmidrule(lr){2-5}\cmidrule(l){6-9}
\textbf{Method} & {\textbf{Success} (\%)} & {\textbf{P50} (s)} & {\textbf{P95} (s)} & {\textbf{Speedup}$^{\dagger}$} & {\textbf{Success} (\%)} & {\textbf{P50} (s)} & {\textbf{P95} (s)} & {\textbf{Speedup}$^{\dagger}$} \\
\midrule
\tgroup{9}{BFCL --- fixed order}
DINGO & 90.3 & 54.300 & 132.900 & 92.9$\times$ & 91.7 & 38.452 & 115.412 & 140.4$\times$ \\
NFA & 82.5 & 61.300 & 148.500 & 113.4$\times$ & 86.0 & 45.522 & 122.171 & 167.4$\times$ \\
DINGO-Opt$^{\ast}$ & \bfseries 100.0 & 2.210 & 5.640 & 3.92$\times$ & \bfseries 100.0 & 1.200 & 3.290 & 3.88$\times$ \\
\ours
\textbf{DConstrain} & \bfseries 100.0 & \bfseries 0.581 & \bfseries 1.481 & \NA & \bfseries 100.0 & \bfseries 0.314 & \bfseries 0.820 & \NA \\
\tgroup{9}{BFCL --- any order}
DINGO-Opt$^{\ast}$ & \multicolumn{1}{c}{71.30} & 5.85 & 12.25 & 9.43$\times$ & \multicolumn{1}{c}{85.81} & 3.76 & 14.16 & 10.39$\times$ \\
\ours
\textbf{DConstrain} & \bfseries 100.0 & \bfseries 0.66 & \bfseries 1.23 & \NA & \bfseries 100.0 & \bfseries 0.41 & \bfseries 0.90 & \NA \\
\tgroup{9}{JSB --- fixed order}
DINGO & 67.7 & 38.003 & 129.107 & 85.0$\times$ & 67.7 & 31.972 & 96.340 & 122.9$\times$ \\
NFA & 63.7 & 45.261 & 128.400 & 102.8$\times$ & 71.3 & 30.714 & 113.846 & 114.0$\times$ \\
DINGO-Opt$^{\ast}$ & 74.8 & 2.510 & 70.500 & 4.83$\times$ & 77.6 & 1.650 & 37.660 & 4.53$\times$ \\
\ours
\textbf{DConstrain} & \bfseries 81.4 & \bfseries 0.602 & \bfseries 14.137 & \NA & \bfseries 83.3 & \bfseries 0.367 & \bfseries 8.265 & \NA \\
\tgroup{9}{CFB --- fixed order}
DINGO & 34.8 & 115.180 & 159.237 & 121.9$\times$ & 44.7 & 84.733 & 124.044 & 174.6$\times$ \\
NFA & 32.5 & 131.287 & 175.143 & 132.0$\times$ & 74.0 & 116.627 & 171.694 & 221.9$\times$ \\
DINGO-Opt$^{\ast}$ & \bfseries 100.0 & 4.657 & 7.340 & 4.02$\times$ & \bfseries 100.0 & 2.210 & 3.020 & 3.96$\times$ \\
\ours
\textbf{DConstrain} & \bfseries 100.0 & \bfseries 1.170 & \bfseries 1.782 & \NA & \bfseries 100.0 & \bfseries 0.560 & \bfseries 0.758 & \NA \\
\tgroup{9}{CFB --- any order}
DINGO-Opt$^{\ast}$ & 9.9 & 7.63 & 16.69 & 8.12$\times$ & 13.7 & 2.90 & 9.40 & 3.54$\times$ \\
\ours
\textbf{DConstrain} & \bfseries 100.0 & \bfseries 0.99 & \bfseries 1.41 & \NA & \bfseries 100.0 & \bfseries 0.80 & \bfseries 1.16 & \NA \\
\bottomrule
\end{tabular}
\end{adjustbox}
\end{table*}
\subsection{Main Results}
\label{sec:compilation_results}

Table~\ref{tab:cold_compilation} evaluates cold compilation under fixed and arbitrary property order, with nesting depth limited to four for comparability with the finite-state baselines. The BFCL runtime evaluation uses two complementary cohorts. Table~\ref{tab:runtime_comparison} counts cold admission, runtime completion, strict-schema validity, and official semantic accuracy over all 2,495 cases; TPS and TPF use each row's completed cases. Table~\ref{tab:runtime_common_success} defines a common runtime-success cohort independently for each model and property-order condition by intersecting the four reported methods within that condition. Schema-invalid completions remain in each cohort. DINGO-Opt compiles 71.30\% of BFCL schemas with LLaDA2.2 and 85.81\% with DiffusionGemma under Any-order. Table~\ref{tab:runtime_comparison} reports its Any-order outcomes on the full pool. Including it in the intersection would substantially shrink the common-success cohorts in Table~\ref{tab:runtime_common_success}; Table~\ref{tab:runtime_latency_matched} compares DINGO-Opt and DConstrain on their paired successful cases. Each order condition has its own common-success cohort. Appendix~\ref{app:compilation_protocol} specifies the accounting rules.

\paragraph{Cold-compilation efficiency.}
DConstrain reduces compilation cost across both tokenizers and all datasets. Under fixed order and depth four, it is $3.88$--$4.83\times$ faster than DINGO-Opt on paired cases, even though DINGO-Opt already shares our C++ pipeline and CSR representation (Appendix~\ref{app:runtime_integration}). The gap widens under arbitrary property order: DConstrain keeps 100\% compilation coverage on BFCL and CFB, whereas DINGO-Opt loses substantial coverage. On the cases that both methods compile, DConstrain remains $3.54$--$10.39\times$ faster. On JSB, most remaining DConstrain failures with the LLaDA2.2 tokenizer come from JSON Schema features outside the supported fragment, such as \texttt{oneOf} exclusivity and general \texttt{allOf} intersection (13.7\% of all schemas); 4.8\% exceed compilation resource or construction budgets (Appendix~\ref{app:jsb_failure_analysis}). Mature engines also support only part of this benchmark: on its GitHub-Hard split, XGrammar and Guidance declare support for 69\% and 60\% of the schemas~\citep{geng2025jsonschemabench}.

\paragraph{Runtime completion and throughput.}
DConstrain completes all 2,495 cases with 100\% strict-schema validity under both models and both property-order settings (Table~\ref{tab:runtime_comparison}). On the common cohorts in Table~\ref{tab:runtime_common_success}, it has the highest TPS among the constrained methods in all four settings. Its TPS is $1.56$--$17.03\times$ that of LAVE-Opt and $13.15$--$30.21\times$ that of EPIC-Opt. Under Schema-order, it reduces overhead above Native-Opt by $10.14\times$ for LLaDA2.2 and $21.75\times$ for DiffusionGemma relative to DINGO-Opt. On the paired Any-order cohorts in Table~\ref{tab:runtime_latency_matched}, DConstrain reduces end-to-end latency by $6.65\times$ and $46.26\times$, and solve latency per forward by $26.40\times$ and $108.54\times$, respectively.

Table~\ref{tab:runtime_latency_matched} reports end-to-end and per-forward solver latency on the DConstrain/DINGO-Opt common-success cohorts. Warm compilation overlaps the model forward pass. In every plotted setting of Figure~\ref{fig:warm_compile_overlap}, at least 94\% of warm-compilation calls finish within the mean forward time. Appendix~\ref{app:tree_sampling} reports additional solver experiments.

\paragraph{Tokens per forward and GPU computation.}
TPF measures how many model-generated tokens are produced per generation forward, making it an important measure of GPU computation per output token in a DLM. Raising the token-acceptance threshold commits fewer tokens per step and can improve accuracy through additional refinement, reducing TPF. LAVE also uses additional model forwards for lookahead to reduce uncertainty~\citep{zhang2026lave}. Both approaches spend more GPU computation per generated token. DConstrain combines the highest Any-order semantic accuracy with higher TPF than the verification-based methods on both models: 81.60\% accuracy at 8.10 TPF for LLaDA2.2 and 81.56\% at 8.15 TPF for DiffusionGemma on the full pool. The common-success results show the same combination of high accuracy and high TPF. TPF counts non-stop model token IDs over all generation forwards for the selected cases; CPU-generated completion text is excluded.

\begin{figure*}[t]
\centering
\includegraphics[width=0.98\textwidth]{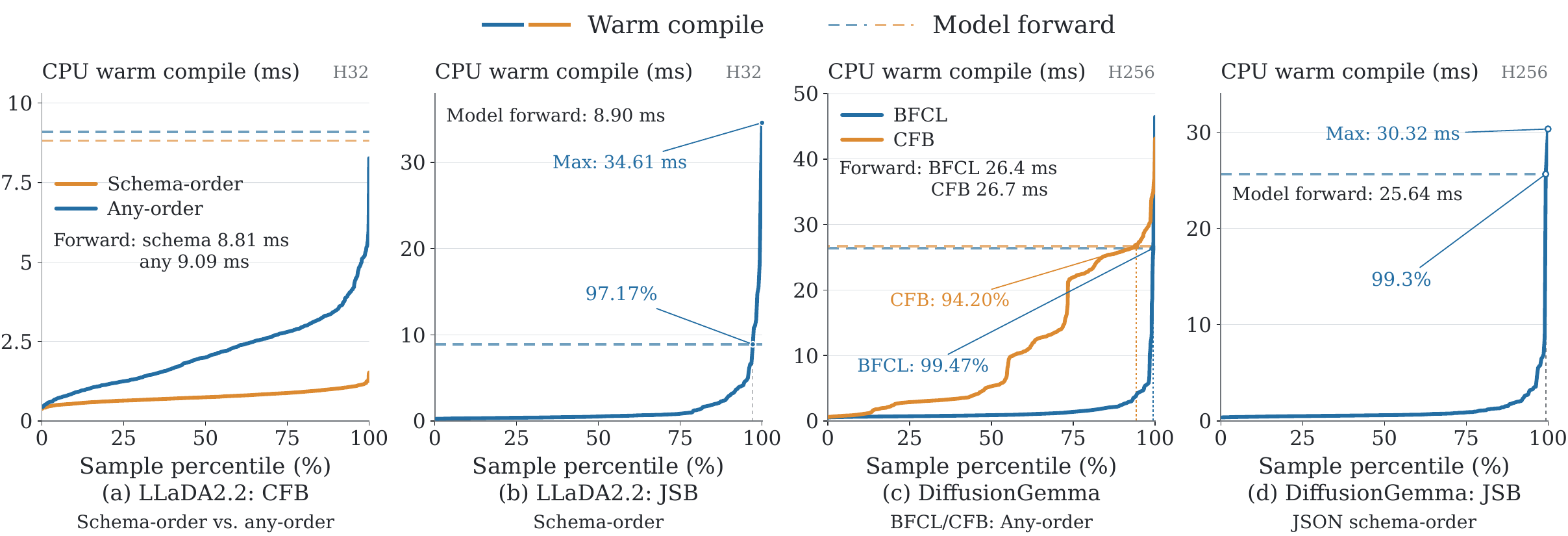}
\caption{Warm-compilation latency under asynchronous execution. CPU warm compilation overlaps with GPU model execution. Dashed lines indicate mean model-forward durations.}
\label{fig:warm_compile_overlap}
\end{figure*}

\begin{table*}[t]
\centering
\setlength{\abovecaptionskip}{0pt}
\setlength{\belowcaptionskip}{7pt}
\caption{BFCL outcomes over all 2,495 cases, reported separately for Schema-order and Any-order. $N$ counts each method's runtime completions. Cold admission, runtime completion (Run), strict-schema validity (Valid.), and official semantic accuracy (Acc.) use 2,495 as the denominator; TPS and TPF use each row's $N$ completed cases.}
\label{tab:runtime_comparison}
\small
\setlength{\tabcolsep}{1.8pt}
\renewcommand{\arraystretch}{1.16}
\begin{adjustbox}{max width=\textwidth}
\begin{tabular}{@{}lcccccccc@{\hspace{10pt}}cccccccc@{}}
\toprule
\multicolumn{17}{@{}l}{\textbf{Schema-order}} \\[3pt]
& \multicolumn{8}{c}{\textbf{DiffusionGemma}} & \multicolumn{8}{c}{\textbf{LLaDA2.2}} \\
\cmidrule(lr){2-9}\cmidrule(l){10-17}
Method & {$N$} & {\shortstack{Cold\\(\%)}} & {\shortstack{Run\\(\%)}} & {\shortstack{Valid.\\(\%)}} & {TPF} & {\shortstack{Acc.\\(\%)}} & {TPS} & {\shortstack{Overhead\\(\%) $\downarrow$}} & {$N$} & {\shortstack{Cold\\(\%)}} & {\shortstack{Run\\(\%)}} & {\shortstack{Valid.\\(\%)}} & {TPF} & {\shortstack{Acc.\\(\%)}} & {TPS} & {\shortstack{Overhead\\(\%) $\downarrow$}} \\
\midrule
Native-Opt & \multicolumn{1}{c}{\textbf{2495}} & \NA & \multicolumn{1}{c}{\textbf{100.00}} & \multicolumn{1}{c}{89.50} & \multicolumn{1}{c}{8.04} & \multicolumn{1}{c}{\textbf{76.39}} & \multicolumn{1}{c}{\textbf{165.14}} & \multicolumn{1}{c}{\textbf{0.0}} & \multicolumn{1}{c}{\textbf{2495}} & \NA & \multicolumn{1}{c}{\textbf{100.00}} & \multicolumn{1}{c}{94.87} & \multicolumn{1}{c}{\textbf{7.40}} & \multicolumn{1}{c}{\textbf{80.80}} & \multicolumn{1}{c}{\textbf{752.25}} & \multicolumn{1}{c}{\textbf{0.0}} \\
\midrule
\tgroup{17}{Verification-based}
LAVE-Opt & \multicolumn{1}{c}{2272} & \multicolumn{1}{c}{\textbf{100.00}} & \multicolumn{1}{c}{91.06} & \multicolumn{1}{c}{91.06} & \multicolumn{1}{c}{6.05} & \multicolumn{1}{c}{36.71\textsubscript{\,\textcolor{red!75!black}{$\downarrow$39.68}}} & \multicolumn{1}{c}{6.00} & \multicolumn{1}{c}{2652.3} & \multicolumn{1}{c}{2480} & \multicolumn{1}{c}{\textbf{100.00}} & \multicolumn{1}{c}{99.40} & \multicolumn{1}{c}{99.40} & \multicolumn{1}{c}{4.67} & \multicolumn{1}{c}{72.55\textsubscript{\,\textcolor{red!75!black}{$\downarrow$8.25}}} & \multicolumn{1}{c}{221.51} & \multicolumn{1}{c}{239.6} \\
EPIC-Opt & \multicolumn{1}{c}{2024} & \multicolumn{1}{c}{\textbf{100.00}} & \multicolumn{1}{c}{81.12} & \multicolumn{1}{c}{80.92} & \multicolumn{1}{c}{2.94} & \multicolumn{1}{c}{36.63\textsubscript{\,\textcolor{red!75!black}{$\downarrow$39.76}}} & \multicolumn{1}{c}{3.43} & \multicolumn{1}{c}{4714.6} & \multicolumn{1}{c}{2463} & \multicolumn{1}{c}{98.88} & \multicolumn{1}{c}{98.72} & \multicolumn{1}{c}{98.72} & \multicolumn{1}{c}{5.62} & \multicolumn{1}{c}{63.41\textsubscript{\,\textcolor{red!75!black}{$\downarrow$17.39}}} & \multicolumn{1}{c}{24.80} & \multicolumn{1}{c}{2933.3} \\
\midrule
\tgroup{17}{Automata-based}
DINGO-Opt & \multicolumn{1}{c}{\textbf{2495}} & \multicolumn{1}{c}{\textbf{100.00}} & \multicolumn{1}{c}{\textbf{100.00}} & \multicolumn{1}{c}{\textbf{100.00}} & \multicolumn{1}{c}{8.11} & \multicolumn{1}{c}{37.19\textsubscript{\,\textcolor{red!75!black}{$\downarrow$39.20}}} & \multicolumn{1}{c}{23.13} & \multicolumn{1}{c}{614.0} & \multicolumn{1}{c}{\textbf{2495}} & \multicolumn{1}{c}{\textbf{100.00}} & \multicolumn{1}{c}{\textbf{100.00}} & \multicolumn{1}{c}{\textbf{100.00}} & \multicolumn{1}{c}{6.86} & \multicolumn{1}{c}{72.63\textsubscript{\,\textcolor{red!75!black}{$\downarrow$8.17}}} & \multicolumn{1}{c}{101.32} & \multicolumn{1}{c}{642.4} \\
\ours \textbf{DConstrain} & \multicolumn{1}{c}{\textbf{2495}} & \multicolumn{1}{c}{\textbf{100.00}} & \multicolumn{1}{c}{\textbf{100.00}} & \multicolumn{1}{c}{\textbf{100.00}} & \multicolumn{1}{c}{\textbf{8.13}} & \multicolumn{1}{c}{37.27\textsubscript{\,\textcolor{red!75!black}{$\downarrow$39.12}}} & \multicolumn{1}{c}{130.47\textsubscript{\,\textcolor{spdgreen}{5.64$\times$}}} & \multicolumn{1}{c}{\textbf{26.6}\textsubscript{\,23.10$\times$}} & \multicolumn{1}{c}{\textbf{2495}} & \multicolumn{1}{c}{\textbf{100.00}} & \multicolumn{1}{c}{\textbf{100.00}} & \multicolumn{1}{c}{\textbf{100.00}} & \multicolumn{1}{c}{6.89} & \multicolumn{1}{c}{72.71\textsubscript{\,\textcolor{red!75!black}{$\downarrow$8.09}}} & \multicolumn{1}{c}{437.94\textsubscript{\,\textcolor{spdgreen}{4.32$\times$}}} & \multicolumn{1}{c}{\textbf{71.8}\textsubscript{\,8.95$\times$}} \\
\midrule
\addlinespace[4pt]
\multicolumn{17}{@{}l}{\textbf{Any-order}} \\[3pt]
& \multicolumn{8}{c}{\textbf{DiffusionGemma}} & \multicolumn{8}{c}{\textbf{LLaDA2.2}} \\
\cmidrule(lr){2-9}\cmidrule(l){10-17}
Method & {$N$} & {\shortstack{Cold\\(\%)}} & {\shortstack{Run\\(\%)}} & {\shortstack{Valid.\\(\%)}} & {TPF} & {\shortstack{Acc.\\(\%)}} & {TPS} & {\shortstack{Overhead\\(\%) $\downarrow$}} & {$N$} & {\shortstack{Cold\\(\%)}} & {\shortstack{Run\\(\%)}} & {\shortstack{Valid.\\(\%)}} & {TPF} & {\shortstack{Acc.\\(\%)}} & {TPS} & {\shortstack{Overhead\\(\%) $\downarrow$}} \\
\midrule
Native-Opt & \multicolumn{1}{c}{\textbf{2495}} & \NA & \multicolumn{1}{c}{\textbf{100.00}} & \multicolumn{1}{c}{90.34} & \multicolumn{1}{c}{8.14} & \multicolumn{1}{c}{76.55} & \multicolumn{1}{c}{\textbf{167.90}} & \multicolumn{1}{c}{\textbf{0.0}} & \multicolumn{1}{c}{\textbf{2495}} & \NA & \multicolumn{1}{c}{\textbf{100.00}} & \multicolumn{1}{c}{94.95} & \multicolumn{1}{c}{7.21} & \multicolumn{1}{c}{80.68} & \multicolumn{1}{c}{\textbf{736.38}} & \multicolumn{1}{c}{\textbf{0.0}} \\
\midrule
\tgroup{17}{Verification-based}
LAVE-Opt & \multicolumn{1}{c}{2408} & \multicolumn{1}{c}{97.07} & \multicolumn{1}{c}{96.51} & \multicolumn{1}{c}{96.47} & \multicolumn{1}{c}{7.71} & \multicolumn{1}{c}{79.20\textsubscript{\,\textcolor{spdgreen}{$\uparrow$2.65}}} & \multicolumn{1}{c}{9.50} & \multicolumn{1}{c}{1667.4} & \multicolumn{1}{c}{2420} & \multicolumn{1}{c}{97.07} & \multicolumn{1}{c}{96.99} & \multicolumn{1}{c}{96.99} & \multicolumn{1}{c}{5.74} & \multicolumn{1}{c}{79.24\textsubscript{\,\textcolor{red!75!black}{$\downarrow$1.44}}} & \multicolumn{1}{c}{313.71} & \multicolumn{1}{c}{134.7} \\
EPIC-Opt & \multicolumn{1}{c}{2212} & \multicolumn{1}{c}{92.38} & \multicolumn{1}{c}{88.66} & \multicolumn{1}{c}{88.18} & \multicolumn{1}{c}{7.44} & \multicolumn{1}{c}{74.71\textsubscript{\,\textcolor{red!75!black}{$\downarrow$1.84}}} & \multicolumn{1}{c}{6.28} & \multicolumn{1}{c}{2573.6} & \multicolumn{1}{c}{2151} & \multicolumn{1}{c}{86.41} & \multicolumn{1}{c}{86.21} & \multicolumn{1}{c}{86.21} & \multicolumn{1}{c}{6.09} & \multicolumn{1}{c}{65.05\textsubscript{\,\textcolor{red!75!black}{$\downarrow$15.63}}} & \multicolumn{1}{c}{37.35} & \multicolumn{1}{c}{1871.6} \\
\midrule
\tgroup{17}{Automata-based}
DINGO-Opt & \multicolumn{1}{c}{2141} & \multicolumn{1}{c}{85.81} & \multicolumn{1}{c}{85.81} & \multicolumn{1}{c}{85.81} & \multicolumn{1}{c}{8.13} & \multicolumn{1}{c}{69.85\textsubscript{\,\textcolor{red!75!black}{$\downarrow$6.70}}} & \multicolumn{1}{c}{2.96} & \multicolumn{1}{c}{5572.3} & \multicolumn{1}{c}{1779} & \multicolumn{1}{c}{71.30} & \multicolumn{1}{c}{71.30} & \multicolumn{1}{c}{71.30} & \multicolumn{1}{c}{\textbf{8.11}} & \multicolumn{1}{c}{59.30\textsubscript{\,\textcolor{red!75!black}{$\downarrow$21.38}}} & \multicolumn{1}{c}{42.39} & \multicolumn{1}{c}{1637.2} \\
\ours \textbf{DConstrain} & \multicolumn{1}{c}{\textbf{2495}} & \multicolumn{1}{c}{\textbf{100.00}} & \multicolumn{1}{c}{\textbf{100.00}} & \multicolumn{1}{c}{\textbf{100.00}} & \multicolumn{1}{c}{\textbf{8.15}} & \multicolumn{1}{c}{\textbf{81.56}\textsubscript{\,\textcolor{spdgreen}{$\uparrow$5.01}}} & \multicolumn{1}{c}{103.82\textsubscript{\,\textcolor{spdgreen}{35.07$\times$}}} & \multicolumn{1}{c}{\textbf{61.7}\textsubscript{\,90.28$\times$}} & \multicolumn{1}{c}{\textbf{2495}} & \multicolumn{1}{c}{\textbf{100.00}} & \multicolumn{1}{c}{\textbf{100.00}} & \multicolumn{1}{c}{\textbf{100.00}} & \multicolumn{1}{c}{8.10} & \multicolumn{1}{c}{\textbf{81.60}\textsubscript{\,\textcolor{spdgreen}{$\uparrow$0.92}}} & \multicolumn{1}{c}{386.70\textsubscript{\,\textcolor{spdgreen}{9.12$\times$}}} & \multicolumn{1}{c}{\textbf{90.4}\textsubscript{\,18.10$\times$}} \\
\bottomrule
\end{tabular}
\end{adjustbox}
\par\vspace{4pt}
\begin{minipage}{\textwidth}
\footnotesize\raggedright
Bold marks the best values within each model and order panel, including ties, and highlights DConstrain overhead. Acc. subscripts show percentage-point changes from Native-Opt: \textcolor{spdgreen}{$\uparrow$} indicates an increase and \textcolor{red!75!black}{$\downarrow$} a decrease. Overhead is $(\mathrm{TPS}_{\mathrm{Native\text{-}Opt}}/\mathrm{TPS}-1)\times100\%$. DConstrain TPS and Overhead subscripts give throughput speedups and overhead-reduction factors over DINGO-Opt in the same panel, respectively. Table~\ref{tab:runtime_latency_matched} reports the paired latency comparison. Schema-order fixes property order; Any-order permits arbitrary order. For Native-Opt, these identify prompt conditions. TPS is model tokens/s; TPF counts non-stop model token IDs per generation forward, excluding CPU completion text. TPS timing excludes prefill for LLaDA2.2 and includes it for DiffusionGemma. The -Opt suffix denotes our implementation optimizations and decoding adaptations for the baselines.
\end{minipage}
\end{table*}

\begin{table*}[t]
\centering
\setlength{\abovecaptionskip}{0pt}
\setlength{\belowcaptionskip}{7pt}
\caption{BFCL quality and throughput on common runtime-success cohorts, defined separately for each model and property-order condition. Within each panel, we intersect the successful-run sets of Native-Opt, LAVE-Opt, EPIC-Opt, and DConstrain. Schema-invalid completions are retained.}
\label{tab:runtime_common_success}
\small
\setlength{\tabcolsep}{2.5pt}
\renewcommand{\arraystretch}{1.16}
\begin{adjustbox}{max width=\textwidth}
\begin{tabular}{@{}lccccc@{\hspace{10pt}}ccccc@{}}
\toprule
\multicolumn{11}{@{}l}{\textbf{Schema-order}} \\[3pt]
& \multicolumn{5}{c}{\textbf{DiffusionGemma} ($N=1{,}874$)} & \multicolumn{5}{c}{\textbf{LLaDA2.2} ($N=2{,}448$)} \\
\cmidrule(lr){2-6}\cmidrule(l){7-11}
Method & {\shortstack{Valid.\\(\%)}} & {TPF} & {\shortstack{Acc.\\(\%)}} & {TPS} & {\shortstack{Overhead\\(\%) $\downarrow$}} & {\shortstack{Valid.\\(\%)}} & {TPF} & {\shortstack{Acc.\\(\%)}} & {TPS} & {\shortstack{Overhead\\(\%) $\downarrow$}} \\
\midrule
Native-Opt & \multicolumn{1}{c}{89.70} & \multicolumn{1}{c}{7.07} & \multicolumn{1}{c}{\textbf{78.50}} & \multicolumn{1}{c}{\textbf{145.71}} & \multicolumn{1}{c}{\textbf{0.0}} & \multicolumn{1}{c}{94.97} & \multicolumn{1}{c}{\textbf{7.42}} & \multicolumn{1}{c}{\textbf{81.04}} & \multicolumn{1}{c}{\textbf{754.40}} & \multicolumn{1}{c}{\textbf{0.0}} \\
\midrule
\tgroup{11}{Verification-based}
LAVE-Opt & \multicolumn{1}{c}{\textbf{100.00}} & \multicolumn{1}{c}{5.89} & \multicolumn{1}{c}{48.19\textsubscript{\,\textcolor{red!75!black}{$\downarrow$30.31}}} & \multicolumn{1}{c}{6.76} & \multicolumn{1}{c}{2055.5} & \multicolumn{1}{c}{\textbf{100.00}} & \multicolumn{1}{c}{4.67} & \multicolumn{1}{c}{72.74\textsubscript{\,\textcolor{red!75!black}{$\downarrow$8.30}}} & \multicolumn{1}{c}{221.32} & \multicolumn{1}{c}{240.9} \\
EPIC-Opt & \multicolumn{1}{c}{99.73} & \multicolumn{1}{c}{3.19} & \multicolumn{1}{c}{48.72\textsubscript{\,\textcolor{red!75!black}{$\downarrow$29.78}}} & \multicolumn{1}{c}{3.81} & \multicolumn{1}{c}{3724.4} & \multicolumn{1}{c}{\textbf{100.00}} & \multicolumn{1}{c}{5.63} & \multicolumn{1}{c}{64.65\textsubscript{\,\textcolor{red!75!black}{$\downarrow$16.39}}} & \multicolumn{1}{c}{25.00} & \multicolumn{1}{c}{2917.6} \\
\midrule
\tgroup{11}{Automata-based}
DINGO-Opt & \multicolumn{1}{c}{\textbf{100.00}} & \multicolumn{1}{c}{\textbf{7.13}} & \multicolumn{1}{c}{48.88\textsubscript{\,\textcolor{red!75!black}{$\downarrow$29.62}}} & \multicolumn{1}{c}{21.48} & \multicolumn{1}{c}{578.4} & \multicolumn{1}{c}{\textbf{100.00}} & \multicolumn{1}{c}{6.95} & \multicolumn{1}{c}{72.99\textsubscript{\,\textcolor{red!75!black}{$\downarrow$8.05}}} & \multicolumn{1}{c}{93.03} & \multicolumn{1}{c}{710.9} \\
\ours \textbf{DConstrain} & \multicolumn{1}{c}{\textbf{100.00}} & \multicolumn{1}{c}{\textbf{7.13}} & \multicolumn{1}{c}{48.88\textsubscript{\,\textcolor{red!75!black}{$\downarrow$29.62}}} & \multicolumn{1}{c}{115.10\textsubscript{\,\textcolor{spdgreen}{5.36$\times$}}} & \multicolumn{1}{c}{\textbf{26.6}\textsubscript{\,21.75$\times$}} & \multicolumn{1}{c}{\textbf{100.00}} & \multicolumn{1}{c}{6.95} & \multicolumn{1}{c}{72.99\textsubscript{\,\textcolor{red!75!black}{$\downarrow$8.05}}} & \multicolumn{1}{c}{443.47\textsubscript{\,\textcolor{spdgreen}{4.77$\times$}}} & \multicolumn{1}{c}{\textbf{70.1}\textsubscript{\,10.14$\times$}} \\
\midrule
\addlinespace[4pt]
\multicolumn{11}{@{}l}{\textbf{Any-order}} \\[3pt]
& \multicolumn{5}{c}{\textbf{DiffusionGemma} ($N=2{,}201$)} & \multicolumn{5}{c}{\textbf{LLaDA2.2} ($N=2{,}149$)} \\
\cmidrule(lr){2-6}\cmidrule(l){7-11}
Method & {\shortstack{Valid.\\(\%)}} & {TPF} & {\shortstack{Acc.\\(\%)}} & {TPS} & {\shortstack{Overhead\\(\%) $\downarrow$}} & {\shortstack{Valid.\\(\%)}} & {TPF} & {\shortstack{Acc.\\(\%)}} & {TPS} & {\shortstack{Overhead\\(\%) $\downarrow$}} \\
\midrule
Native-Opt & \multicolumn{1}{c}{91.46} & \multicolumn{1}{c}{\textbf{7.81}} & \multicolumn{1}{c}{79.46} & \multicolumn{1}{c}{\textbf{165.63}} & \multicolumn{1}{c}{\textbf{0.0}} & \multicolumn{1}{c}{95.53} & \multicolumn{1}{c}{7.41} & \multicolumn{1}{c}{82.78} & \multicolumn{1}{c}{\textbf{756.59}} & \multicolumn{1}{c}{\textbf{0.0}} \\
\midrule
\tgroup{11}{Verification-based}
LAVE-Opt & \multicolumn{1}{c}{\textbf{100.00}} & \multicolumn{1}{c}{7.43} & \multicolumn{1}{c}{83.55\textsubscript{\,\textcolor{spdgreen}{$\uparrow$4.09}}} & \multicolumn{1}{c}{9.92} & \multicolumn{1}{c}{1569.7} & \multicolumn{1}{c}{\textbf{100.00}} & \multicolumn{1}{c}{5.75} & \multicolumn{1}{c}{82.60\textsubscript{\,\textcolor{red!75!black}{$\downarrow$0.18}}} & \multicolumn{1}{c}{315.63} & \multicolumn{1}{c}{139.7} \\
EPIC-Opt & \multicolumn{1}{c}{99.45} & \multicolumn{1}{c}{7.51} & \multicolumn{1}{c}{84.55\textsubscript{\,\textcolor{spdgreen}{$\uparrow$5.09}}} & \multicolumn{1}{c}{6.30} & \multicolumn{1}{c}{2529.0} & \multicolumn{1}{c}{\textbf{100.00}} & \multicolumn{1}{c}{6.09} & \multicolumn{1}{c}{75.48\textsubscript{\,\textcolor{red!75!black}{$\downarrow$7.30}}} & \multicolumn{1}{c}{37.34} & \multicolumn{1}{c}{1926.2} \\
\midrule
\tgroup{11}{Automata-based}
\ours \textbf{DConstrain} & \multicolumn{1}{c}{\textbf{100.00}} & \multicolumn{1}{c}{\textbf{7.81}} & \multicolumn{1}{c}{\textbf{84.87}\textsubscript{\,\textcolor{spdgreen}{$\uparrow$5.41}}} & \multicolumn{1}{c}{110.50} & \multicolumn{1}{c}{\textbf{49.9}} & \multicolumn{1}{c}{\textbf{100.00}} & \multicolumn{1}{c}{\textbf{8.22}} & \multicolumn{1}{c}{\textbf{83.15}\textsubscript{\,\textcolor{spdgreen}{$\uparrow$0.37}}} & \multicolumn{1}{c}{490.98} & \multicolumn{1}{c}{\textbf{54.1}} \\
\bottomrule
\end{tabular}
\end{adjustbox}
\par\vspace{4pt}
\begin{minipage}{\textwidth}
\footnotesize\raggedright
Bold marks the best values within each model and order panel, including ties, and highlights DConstrain overhead. Acc. subscripts show percentage-point changes from Native-Opt: \textcolor{spdgreen}{$\uparrow$} indicates an increase and \textcolor{red!75!black}{$\downarrow$} a decrease. Overhead is $(\mathrm{TPS}_{\mathrm{Native\text{-}Opt}}/\mathrm{TPS}-1)\times100\%$. Under Schema-order, DConstrain TPS and Overhead subscripts give throughput speedups and overhead-reduction factors over DINGO-Opt, respectively. The Any-order comparison with DINGO-Opt appears in Table~\ref{tab:runtime_latency_matched}. TPS/TPF accounting follows Table~\ref{tab:runtime_comparison}. The -Opt suffix denotes our implementation optimizations and decoding adaptations for the baselines.
\end{minipage}
\end{table*}

\subsection{Ablation Study}
\label{sec:ablation}
\begin{table}[t]
\centering
\setlength{\abovecaptionskip}{0pt}
\setlength{\belowcaptionskip}{10pt}
\caption{Compiler ablations on BFCL with arbitrary property order. Values are means; gray numbers are ratios to the full system.}
\label{tab:compiler_ablation}
\small
\setlength{\tabcolsep}{5pt}
\renewcommand{\arraystretch}{1.15}
\begin{tabular}{@{}lccc@{}}
\toprule
Variant & Cold (s) $\downarrow$ & Solve (ms) $\downarrow$ & Memory (MiB) $\downarrow$ \\
\midrule
\ours \textbf{Full (any order)} & \bfseries 0.72 & \bfseries 3.72 & \bfseries 5.94 \\
w/o Two-Stage Env & 10.52\rat{14.61} & 3.83\rat{1.03} & 20.97\rat{3.53} \\
w/o Local-State Reuse & 1.87\rat{2.60} & 3.75\rat{1.01} & 12.98\rat{2.19} \\
w/o Token Grouping & 0.96\rat{1.33} & 25.32\rat{6.81} & 24.42\rat{4.11} \\
\bottomrule
\end{tabular}
\end{table}

Table~\ref{tab:compiler_ablation} isolates three mechanisms on BFCL with the LLaDA2.2 tokenizer and arbitrary property order. \emph{w/o Two-Stage Env} binds parsing-environment values during cold compilation instead of deferring them to runtime. \emph{w/o Local-State Reuse} disables memoization of transition programs for repeated occurrences of the same local grammar state but still keeps the factorized runtime environment and continuation stack. \emph{w/o Token Grouping} processes tokens with equivalent effects independently. All variants share the same grammar semantics and decoding setup.

The three mechanisms remove different forms of redundancy. Deferring environment specialization gives the largest cold-compilation benefit: eager binding raises mean compilation latency to $14.61\times$ that of the full system. Local-state reuse avoids recompiling identical local transition logic. Token-effect grouping mainly reduces vocabulary-level work in downstream joint inference, and removing it raises mean solve latency to $6.81\times$ that of the full system. A pushdown interpreter also avoids finite-state expansion, so Appendix~\ref{app:direct_pda} further compares DConstrain with Direct-PDA-Opt, which interprets the parser at every reached configuration. Even with token grouping before interpretation, it is $6.43\times$ slower end to end on 300 BFCL cases, because it re-executes the shared transition logic for every configuration.

\subsection{Analysis}
\label{sec:analysis}

\paragraph{Property order changes task accuracy.}
Schema-order is a substantive decoding restriction for DLMs. On the full 2,495-case BFCL pool in Table~\ref{tab:runtime_comparison}, allowing arbitrary property order raises DConstrain's semantic accuracy from 72.71\% to 81.60\% for LLaDA2.2 and from 37.27\% to 81.56\% for DiffusionGemma: gains of 8.89 and 44.29 percentage points, respectively, with 100\% schema validity throughout. Native-Opt accuracy changes by only $-0.12$ and $+0.16$ points under the corresponding prompt conditions. The much larger constrained-decoding gains indicate that fixing the field order excludes schema-valid outputs that the models prefer. These comparisons use the same full pool; the separately selected cohorts in Table~\ref{tab:runtime_common_success} support within-order comparisons. The higher LLaDA2.2 TPS in the Any-order common panel reflects its easier case composition: the shared cohort contains 2,149 cases, compared with 2,448 under Schema-order.

\paragraph{Verification, sampling bias, and recovery.}
A verifier answers whether a proposed update leaves some legal completion. Weighted joint inference additionally accounts for the model probability mass of compatible continuations. For a fixed canvas graph, the sampling objective assigns each accepting token-labeled path $\pi$ probability proportional to $\prod_i p_i(v_i)$. An existential completion test supplies the support of this objective, but not the continuation masses needed to sample from it. Exact rejection of complete assignments from a fixed proposal samples that proposal conditioned on validity; local token acceptance, batch selection, and recovery introduce different update rules.

Consider an unambiguous two-position grammar accepting only $aa$ and $bb$, with $p_1(a)=p_2(a)=0.9$ and $p_1(b)=p_2(b)=0.1$. Weighted joint sampling selects $aa$ with probability $0.81/(0.81+0.01)=98.78\%$. A policy that samples the first token and commits any feasible choice selects $aa$ with probability $90\%$, because the second token is then forced. Every completed output is valid, yet the distributions differ. The missing factor is the probability mass of legal continuations, which a Boolean verifier replaces by an existence decision. The same distinction affects MAP decoding: selecting high-confidence feasible proposals does not jointly maximize the product of token probabilities.

Finite witness search adds another source of selection bias. LAVE tests model-sampled lookahead completions and uses cached-prefix recovery~\citep{zhang2026lave}. If a proposal $u$ has proposal probability $q(u)$ and witness-search success probability $a(u)$, its probability among accepted proposals is proportional to $q(u)a(u)$. Thus a feasible proposal can be rejected when its valid witnesses are missed, and acceptance depends on the lookahead distribution and search budget. EPIC uses confidence-based subset selection and shrinking before exact verification; on exhausting the rejection budget, it completes the output using a witness that is not sampled from the model distribution~\citep{jin2026epic}. These policies favor different proposals and continuations from those selected by weighted joint inference.

Verification also pays for rejected work. Independent proposals can conflict across positions, so an unsuccessful check triggers resampling, a smaller batch, or recovery. Parallel batch verification reduces the number of sequential checks, but a failed batch still requires additional selection and validation. This cost depends on how often proposals are accepted, whereas automata-based inference resolves compatibility within the weighted graph. Native DLM refinement can revise temporarily incompatible assignments. Maintaining feasibility of committed updates restricts these denoising trajectories, and proposal filtering and recovery further change which trajectories are followed. On LLaDA2.2's Any-order common cohort, Native-Opt reaches 82.78\% semantic accuracy, LAVE-Opt and EPIC-Opt reach 82.60\% and 75.48\%, and DConstrain reaches 83.15\%. The verification-based methods obtain 100\% schema validity on this cohort while losing semantic accuracy relative to Native-Opt.

\paragraph{Property order has opposite runtime effects.}
DINGO-Opt's solve latency per forward rises from 32.45 to 74.18\,ms for LLaDA2.2 and from 216.72 to 3195.39\,ms for DiffusionGemma when arbitrary order is allowed (Table~\ref{tab:runtime_latency_matched}). Enumerating field histories enlarges the finite-state graph, increasing the transitions processed by Viterbi inference; Figure~\ref{fig:ana} and Appendix~\ref{app:tree_sampling} show the associated state-space and solver costs. LAVE-Opt and EPIC-Opt show the opposite trend: LLaDA2.2 overhead falls from 240.9\% to 139.7\% and from 2917.6\% to 1926.2\%, respectively; DiffusionGemma shows the same direction. In our integrations, reusable schema-specific verification structures are prepared during cold compilation, and both orders use the same online verification pipeline. A broader accepted language admits more proposals and reduces repeated resampling and verification per committed token. TPF increases reflect this change: LAVE-Opt rises from 4.67 to 5.75 on LLaDA2.2, and EPIC-Opt from 3.19 to 7.51 on DiffusionGemma.

\paragraph{Efficient joint inference under arbitrary order.}
DConstrain makes arbitrary-order joint inference practical. Under Any-order, its overhead is 54.1\% for LLaDA2.2 and 49.9\% for DiffusionGemma on the common cohorts. It achieves the highest TPS among the evaluated constrained methods in every model--order setting and the highest Any-order semantic accuracy on the common cohorts, at 83.15\% and 84.87\%. On the full pool, it reaches 81.60\% and 81.56\% with complete coverage and 100\% schema validity. The paired comparison in Table~\ref{tab:runtime_latency_matched} shows $6.65\times$ and $46.26\times$ lower end-to-end latency than DINGO-Opt under Any-order. The following analyses connect these results to the compiled representation and its scaling with runtime context.

\paragraph{Factorized representation.}
Table~\ref{tab:compiled_representation_llada} compares the stored representations. Finite-state compilation materializes context-expanded states and token-level transitions, whereas DConstrain stores the local RTN structure and symbolic transition programs. Its median recorded payload is 4.74--5.11\,MB across the three datasets, compared with 36.86--90.61\,MB for DINGO. The runtime experiments measure the cost of deferring work to warm compilation.

\begin{table*}[t]
\centering
\setlength{\abovecaptionskip}{0pt}
\setlength{\belowcaptionskip}{10pt}
\caption{Compiled representation statistics with the LLaDA2.2 tokenizer.}
\label{tab:compiled_representation_llada}
\small
\setlength{\tabcolsep}{4pt}
\renewcommand{\arraystretch}{1.15}
\begin{tabular}{@{}llN{3.2}N{4.2}N{4.2}N{3.2}N{3.2}@{}}
\toprule
Dataset & Method & {States} & \multicolumn{2}{c}{Edges ($10^3$)} & \multicolumn{2}{c}{Recorded payload (MB)} \\
\cmidrule(lr){4-5}\cmidrule(l){6-7}
& & {Mean} & {P50} & {P95} & {P50} & {P95} \\
\midrule
BFCL & DINGO & 228.74 & 524.767 & 1226.268 & 54.20 & 114.91 \\
& NFA & 233.74 & 352.707 & 1052.077 & 31.95 & 93.55 \\
\ours & \textbf{DConstrain} & 184.00 & 3.450 & 14.632 & \bfseries 4.74 & \bfseries 6.31 \\
\midrule
JSONSchemaBench & DINGO & 182.42 & 350.973 & 1224.937 & 36.86 & 118.83 \\
& NFA & 214.94 & 353.617 & 1054.279 & 36.99 & 130.63 \\
\ours & \textbf{DConstrain} & 163.09 & 3.603 & 264.375 & \bfseries 5.03 & \bfseries 31.35 \\
\midrule
ComplexFuncBench & DINGO & 404.39 & 1053.460 & 1405.177 & 90.61 & 107.51 \\
& NFA & 409.39 & 1053.276 & 1405.165 & 121.60 & 139.51 \\
\ours & \textbf{DConstrain} & 309.13 & 13.077 & 20.621 & \bfseries 5.11 & \bfseries 5.94 \\
\bottomrule
\end{tabular}
\end{table*}

\begin{figure}[!t]
\centering
\includegraphics[width=\columnwidth]{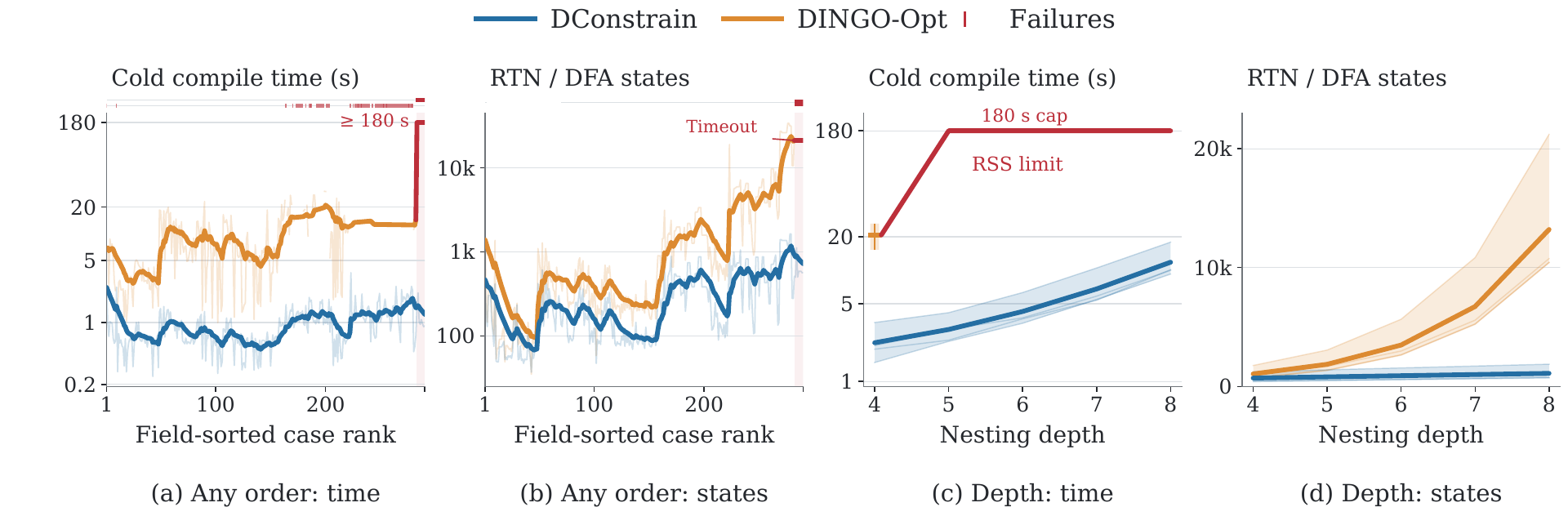}
\caption{Scaling with runtime-context complexity on a fixed 300-case BFCL subset. (a,b) Arbitrary-order objects are ranked by maximum parameter-object field count; (c,d) nesting depth varies from 4 to 8. }
\label{fig:ana}
\end{figure}

\paragraph{Scaling with property order and nesting depth.}
Figure~\ref{fig:ana} varies the two runtime-context dimensions: property order enlarges the parsing environment, and nesting enlarges the continuation stack. As either grows, DINGO-Opt increasingly hits the time or memory limit and fails all evaluated cases at depths 5--8, while DConstrain still compiles them, consistent with Table~\ref{tab:cold_compilation}.

\section{Conclusion}
\label{sec:conclusion}
We presented DConstrain, a compiler for constrained decoding in diffusion language models. Finite-state compilers expand runtime contexts into states. For unordered properties, this blow-up is exponential even for context-free grammars, so existing exact decoders fix the property order. DConstrain instead compiles reusable context-parameterized transition programs, which keep the stack and the field history as runtime parameters and still give joint inference exact successors. Warm compilation and grammar propagation overlap the model forward pass, and token-effect grouping shrinks the graph while preserving the ungrouped graph's Viterbi optimum and token-labeled path distribution. Together, these choices substantially improve compilation coverage and latency and reduce constrained-decoding overhead while preserving schema validity and task accuracy. They also make arbitrary property order practical, which further improves task accuracy over fixed-order constraints.

\bibliography{iclr2027_conference}
\bibliographystyle{iclr2027_conference}

\clearpage
\appendix
\section{Additional Background and Theoretical Analysis}
\label{app:theory}

This appendix formalizes the parser interface used in
Section~\ref{sec:preliminaries}, establishes correctness of the compiler
in Section~\ref{sec:methodology}, establishes weighted inference preservation under token-effect grouping, and analyzes compilation and graph-construction costs.

\subsection{Grammar Automata and Parser Configurations}
\label{app:grammar_background}

Let $\mathcal G$ define an output language $\mathcal L_{\mathcal G}$.
Finite-state constraints admit an automaton
$\mathcal M_{\mathcal G}=(Q,\Sigma,\delta_{\mathrm{FA}},q_0,F)$, with transitions
$\delta_{\mathrm{FA}}:Q\times\Sigma\to 2^Q$; a DFA has at most one successor per
state--symbol pair, whereas an NFA may have several. For nested grammars,
a pushdown representation augments local control with a runtime
stack~\citep{dong2025xgrammar}. Its logical configuration is
\begin{equation}
c=(q,s,e),
\label{eq:parser_configuration}
\end{equation}
where $s$ stores continuations and $e$ contains parsing variables.
Guards, stack operations, and environment updates allow the same local
transition logic to operate across different contexts.

Token-level transitions consume tokenizer tokens rather than individual
grammar symbols~\citep{willard2023guided,dong2025xgrammar}. Let
$\mathsf T^{\mathrm{raw}}_{\mathcal G}(c,v)$ contain all successors after
consuming $v$, including required stack, environment, and token-boundary
operations. We use the normalized relation
\begin{equation}
\mathsf T_{\mathcal G}(c,v)
:=\mathsf T^N_{\mathcal G}(c,v)
=N\!\left(\mathsf T^{\mathrm{raw}}_{\mathcal G}(c,v)\right),
\label{eq:token_transition_relation}
\end{equation}
where $N$ removes statically certified dead configurations, as defined
below. Relations extend to frontiers by union over their configurations.
For finite automata, $c$ reduces to a state $q$.

\subsection{Diffusion Language Models}
\label{app:dlm_background}

A partially resolved window
$\boldsymbol x^t\in(\mathcal V\cup\{\mathtt{[MASK]}\})^H$
follows a fixed prefix $\boldsymbol u$. At denoising iteration $t$, the
model predicts distributions at all positions~\citep{sahoo2024mdlm,nie2025llada}:
\begin{equation}
\begin{aligned}
P_t&=\operatorname{Model}_{\theta}(\boldsymbol u,\boldsymbol x^t),\\
P_t[i,v]&=p_{\theta}(y_i=v\mid\boldsymbol u,\boldsymbol x^t).
\end{aligned}
\label{eq:dlm_prediction}
\end{equation}
A decoding policy selects updates for the next iteration. Multiple
unresolved positions permit joint grammar-constrained token
selection~\citep{suresh2025dingo,dang2026constrained}. Let
$\mathcal V_i^t$ be the admitted tokens at position $i$; fixed positions
can use singleton sets. All comparisons below use identical candidate
sets, prefixes, protocols, and boundary conditions.

\subsection{Correctness of Factorized Compilation}
\label{sec:factorized_correctness}

We expose $e=(d,\eta)$, where $d$ is container depth and $\eta$ is the
choice environment, and write $c=(q,s,d,\eta)$. Execution states also
retain input, cursor, and protocol phase. A token-complete state has
finished all required boundary processing; an exhausted cursor alone is
insufficient. We assume agreement only between corresponding primitive
parser operations. All results use set semantics and concern completed
computations; resource exhaustion is an error, not an empty result.
Byte execution, compilation, and assembly are defined for every finite
$w\in\mathbb B^*$ over the byte alphabet $\mathbb B$. Token compilation
specializes $w=\operatorname{bytes}(v)$ with its initialization/finalization.
Suffixes need not encode vocabulary tokens and retain the enclosing
metadata and phase without fresh token initialization.

\paragraph{Static normalization.}
The abstract RTN soundly overapproximates concrete execution by relaxing
environment guards while preserving call/return structure. For the set
$\operatorname{Ret}(q)$ of finite abstract derivations returning from $q$'s
current rule, define
\begin{equation}
D(q)=\min_{\pi\in\operatorname{Ret}(q)}|\pi|_{\mathrm{bytes}},
\qquad \min\varnothing=\infty.
\label{eq:static_distance_definition}
\end{equation}
Rule exits, including complete acceptance as a distinguished root exit,
are productive terminals of distance zero. Consequently,
\begin{equation}
\begin{aligned}
D(q)&\le D(q') &&\text{for zero-byte edges},\\
D(q)&\le 1+D(q') &&\text{for byte edges},\\
D(q)&\le D(a)+D(r) &&\text{for calls to $a$ returning to $r$},\\
D(q)&=0 &&\text{at rule exits}.
\end{aligned}
\label{eq:static_distance_constraints}
\end{equation}
Min-plus relaxation starts with zero at terminals and infinity elsewhere.
Each finite value has a derivation witness, and saturation covers every
finite return, computing Eq.~\eqref{eq:static_distance_definition}.
Thus $D(q)=\infty$ excludes any concrete return; finite distance need not
imply concrete feasibility.
Define
\begin{equation}
P(c)=[D(q)<\infty]\land\bigwedge_{r\in s}[D(r)<\infty],
\qquad N(X)=\{c\in X:P(c)\},
\label{eq:productivity_normalization}
\end{equation}
Complete-output accepting configurations satisfy $P$ by their terminal
distance and empty stack. Protocol steps obey the same abstract control
rules and do not bypass required returns. Infinity is used only after the
distance computation completes.

\begin{lemma}[Normalization safety]
\label{lem:normalization_commutation}
For raw byte execution $\delta_{\mathrm{raw}}$ with its prescribed
zero-byte closures,
\begin{equation}
N\!\left(\delta_{\mathrm{raw}}(N(X),w)\right)
=N\!\left(\delta_{\mathrm{raw}}(X,w)\right).
\label{eq:normalization_commutation}
\end{equation}
The same property holds for closures and complete token execution,
including normalization at intermediate stages.
\end{lemma}
\begin{proof}
Every primitive step satisfies
$c\to c'\land P(c')\Rightarrow P(c)$.
For local and byte edges this follows from the distance inequalities.
For calls, productivity of the successor implies finite distances for
the callee and pushed return address. For returns, the source is an exit,
and the popped address and remaining stack have finite distances.
Consequently, every finite execution ending in a retained configuration
contains only retained configurations. Such a path survives initial and
intermediate normalization; the opposite inclusion follows from
$N(X)\subseteq X$.
\end{proof}

Write $\delta_N(X,w)=N(\delta_{\mathrm{raw}}(X,w))$.
Unlike the raw relation, this semantics excludes token endpoints that
are statically known never to complete the output.

\paragraph{Representations and admissibility.}
A parameterized bundle $B=(\mathbf a,\mathbf R,\eta_0)$ contains atoms
$a_i=(q_i,d_i,\Omega_i)$ and their own stack relations $R_i$.
Its logical interpretation is
\begin{equation}
\mathcal D(B)=\bigcup_i
\{(q_i,s,d_i,\Omega_i(\eta_0)):
  s\in R_i,\ \Omega_i(\eta_0)\ne\bot\}.
\label{eq:bundle_decoding}
\end{equation}
Here $\Omega_i$ is an ordered program of partial choice operations.
A public representation is \emph{admissible} when its atom--stack--program
associations are preserved; control, frames, operands, and cursor/protocol
metadata are well formed in their binding scope; and its decoded frontier
is normalized and in entry-closure form. Initialization creates aligned
atoms and applies closure and $N$. The transfers below preserve these
invariants; internal seeds carry their specified suspended phases.

Updates modify the current environment; a non-committing guard obeys
\begin{equation}
\mathsf{check}_{\{U_j\}}(\eta)=
\begin{cases}
\eta,&\exists j:\ U_j(\eta)\ne\bot,\\
\bot,&\text{otherwise}.
\end{cases}
\label{eq:noncommitting_guard}
\end{equation}
Each trial starts from the same current environment. Ordered composition
satisfies $(U;V)(\eta)=V(U(\eta))$, including failure.

\paragraph{Source-indexed decoding.}
At each token query's entry, label every source $c$ by its immutable origin
$o=c$. Source selectors and stack transformations denote a relation
$\mathcal R(\widehat Y)$ between origins and current configurations. Define
\begin{equation}
\mathcal D_o(\widehat Y)
=\{c':(o,c')\in\mathcal R(\widehat Y)\},
\qquad
\mathcal D(\widehat Y)=\bigcup_o\mathcal D_o(\widehat Y).
\label{eq:source_indexed_decoding}
\end{equation}
Initially $\mathcal R=\{(c,c):c\in\mathcal D(\widehat X)\}$.
Transfers keep the origin while changing the current configuration.
Stack operations act pointwise on the selected atom's input stack, not
just its index. Merging unions these relations; shared suffixes compose
them through the same intermediate configuration, never through a
Cartesian product. Origins are proof annotations on existing selectors
and stack maps, fixed within a query and reset at the next query.

\paragraph{Independent byte execution.}
The byte executor matches source atoms, appends call frames, partitions
returns by their actual top frame, and appends choice operations in order.
Its closure explores all prescribed zero-byte alternatives, deduplicates
full current states with their source associations, and applies $N$ to
each atom's own stack relation. It does not invoke local-effect assembly.

\begin{lemma}[Source-indexed byte simulation]
\label{lem:byte_execution_preservation}
For every origin $o$ and byte string $w$, the specified byte executor with
matching closure conventions satisfies
\begin{equation}
\mathcal D_o\!\left(\widehat\delta_{\mathrm{bytes}}(\widehat Y,w)\right)
=\delta_N\!\left(\mathcal D_o(\widehat Y),w\right).
\label{eq:parameterized_byte_simulation}
\end{equation}
Closure and token-finalization satisfy the analogous equations, and full
public operations return admissible representations.
\end{lemma}
\begin{proof}
Fix $o$. Primitive agreement gives the same byte, stack, and choice
successors for each of its configurations, retaining $o$. Extracting an
identical, equally bound prefix $U$ preserves
$V_i(U(\eta_0))=(U;V_i)(\eta_0)$ and the origin associations.
Merging identical current states unions their origins and shares their
future relation. Newly attached origins receive those successors by
propagation or reuse of that relation; saturation therefore covers every
origin's finite closure paths. Normalization acts per endpoint, so
Lemma~\ref{lem:normalization_commutation} applies to each origin.
Induction over bytes, followed by the same finalization argument, proves
the equation. Transfers preserve scopes and phases; public exit closure
and $N$ restore admissibility.
\end{proof}

\paragraph{Local summaries and worklist propagation.}
For arbitrary $w\in\mathbb B^*$, cold compilation scans bytes and expands
local closures using items $h=(a,\rho,\Delta,m,\Omega,\nu)$.
Here $a$ includes control and execution phase, $\rho$ is the unmatched
local return stack, $\Delta$ is net depth change, $m$ is its minimum
prefix change including zero, and $\nu$ retains additional used metadata.
Known calls and local returns modify $\rho$; choice operations append to
$\Omega$. Initially $\rho=\Omega=\varepsilon$ and $\Delta=m=0$.
A depth change $\zeta$ updates
\begin{equation}
(\Delta,m)\leftarrow
(\Delta+\zeta,\min\{m,\Delta+\zeta\}).
\label{eq:depth_summary_update}
\end{equation}
Every explicit alternative is explored. A return with empty $\rho$
emits a residual before accessing the unknown incoming stack.
An ordinary summary $P=(q',\rho,\Delta,m,\Omega,\nu)$ binds as
\begin{equation}
\operatorname{Bind}(P;s,d,\eta)=
\begin{cases}
\{(q',s\mathbin{\|}\rho,d+\Delta,\Omega(\eta))\},
 &d+m\ge0,\ \Omega(\eta)\ne\bot,\\
\varnothing,&\text{otherwise}.
\end{cases}
\label{eq:summary_binding}
\end{equation}
Frames are appended bottom-to-top. Additional checks retain their primitive
tests and operands; a depth-dependent check other than underflow also
records its relative evaluation offset, or sufficient trace statistics,
in keyed metadata $\nu$. Binding evaluates these tests and preserves the
source selector and input-stack argument. Structural steps have no
unrecorded choice effects; execution phase determines remaining closure.

The worklist keys by $(a,\rho,\Delta,\Omega,\nu)$, excluding only $m$.
Each new key or strict increase of its best $m$ is enqueued. A popped item
is skipped only when its $m$ is stale; otherwise all successors are explored.
Ordinary and residual output tables also retain the largest $m$.

\begin{lemma}[Summary and worklist preservation]
\label{lem:local_summary_preservation}
\label{lem:worklist_preservation}
For every concrete input binding, saturation preserves its local reference
endpoints, cut before the first unknown-input-stack return, and attributes
them to that same source.
\end{lemma}
\begin{proof}
Along a local trace the stack is $s\mathbin{\|}\rho$. If $h_j$ denotes
its cumulative depth change, then $\Delta=h_k$ and $m=\min_j h_j$;
the underflow checks succeed exactly when $d+m\ge0$. Ordered interpretation
gives precisely $\Omega(\eta)$; other retained tests use their recorded
operands and offsets. Each summary therefore has the same endpoint and
failure domain as its witnessing trace, for that fixed input binding.

For equal keys and $m_1\le m_2$, the first success domain is contained in
the second, and outputs agree. A common suffix with summary
$(\Delta_U,m_U)$ yields
\begin{equation}
(\Delta,m);(\Delta_U,m_U)
=(\Delta+\Delta_U,\min\{m,\Delta+m_U\}),
\label{eq:depth_summary_composition}
\end{equation}
which is monotone in $m$. Dominance therefore persists through successors.
Every retained item has a trace witness, giving soundness. Conversely,
each prefix of any finite trace is retained or dominated by a retained
item. Re-enqueuing improvements propagates this coverage to outputs at
saturation, giving completeness. The inclusion and propagation argument
holds for each input separately, and binding never changes its origin.
Applying $N$ preserves this sourcewise equality.
\end{proof}

\paragraph{Residual resolution.}
Whole-token replay consumes the original token from its original source,
without first applying the compiled prefix. Source-indexed byte simulation
gives that source's full normalized successor set; union with ordinary
results already contained in that set does not change it or its origin.

Suffix restoration instead applies the prefix once, pops the actual
source-stack return frame once, and continues in that caller with its
retained environment, input, and protocol phase. The compiler's initial
residual set is empty; new records appear only after a byte is consumed,
with cursor equal to its zero-based index plus one. Propagation preserves
that cursor, and empty input generates only the identity effect. Hence
compiler-generated residuals satisfy
\begin{equation}
1\le j\le |w|,\qquad |w[j:]|<|w|.
\label{eq:generated_residual_progress}
\end{equation}
Entry and caller closures handle zero-byte returns independently.

\begin{lemma}[Suffix assembly preservation]
\label{lem:stack_return_continuation}
For every finite byte string $w$, compiler-generated effects, and origin
$o$, assembly with its required closures satisfies
\begin{equation}
\mathcal D_o\!\left(\widehat A^{\mathrm{cl}}(\widehat Y,w)\right)
=\delta_N\!\left(\mathcal D_o(\widehat Y),w\right).
\label{eq:suffix_assembly_preservation}
\end{equation}
\end{lemma}
\begin{proof}
Fix $o$ and induct strongly on $|w|$ over all finite byte strings, with
compatible input and protocol phases. Empty input uses the identity effect
and source-indexed closure simulation. Otherwise, summary and worklist
preservation cover every ordinary endpoint and first residual cut from
this origin. At a cut, prefix binding, the actual return, and caller
closure recover its reference continuation without changing $o$.
The suffix is shorter by Eq.~\eqref{eq:generated_residual_progress}, so
induction applies whether or not it is a tokenizer token. Its results are
composed with the prefix through the same caller configuration.
An empty suffix can return an intermediate seed; the enclosing closure
still processes its pending returns and checks. Union over alternatives
and normalization commutation prove the claim for $o$.
\end{proof}

\paragraph{Token effects and transition preservation.}
Let $\mathcal E(q,v)$ be the token specialization of the byte-string
compiler, including its prescribed initialization/finalization and residual
modes. Define
\begin{equation}
\operatorname{Canon}(\mathcal E)
=\operatorname{Ser}(\operatorname{Red}(\mathcal E)),
\label{eq:structural_canonicalization}
\end{equation}
where $\operatorname{Red}$ applies the proved duplicate and dominance
reductions, and $\operatorname{Ser}$ completely encodes the resulting
alternatives. It retains ordered operations, bindings, and all
\emph{ordinary and residual} token-boundary, initialization/finalization,
and protocol information used by execution. Residuals additionally retain
input, cursor, source bindings, and replay mode. Whole-token selectors
retain token identity. Interned operands are compared
within the same immutable registry and binding scope; hashing must resolve
full key equality. Let $\Pi_q$ partition tokens by these keys.

\begin{lemma}[Token-effect preservation]
\label{lem:token_effect_preservation}
For fixed $q$ and scope, equal canonical effects imply identical normalized
successor sets for each admissible concrete source, with the same source
attribution.
\end{lemma}
\begin{proof}
The reductions preserve unions by the summary and worklist lemma.
Structural equality identifies the same ordinary bindings, residual
computations, and finalization for a fixed source. Their completed outcomes
agree and retain that source. This is a sufficient equivalence test,
not a complete semantic-equivalence procedure.
\end{proof}

\begin{theorem}[Sourcewise normalized token-transition preservation]
\label{thm:factorized_token_preservation}
For an admissible public source $\widehat X$ and compiler-generated effects,
completed factorized token execution returns an admissible representation
and, for every $c\in\mathcal D(\widehat X)$, satisfies
\begin{equation}
\boxed{
\mathcal D_c\!\left(\widehat{\mathsf T}(\widehat X,v)\right)
=\mathsf T_{\mathcal G}(c,v).
}
\label{eq:sourcewise_token_transition_preservation}
\end{equation}
Consequently,
\begin{equation}
\mathcal D\!\left(\widehat{\mathsf T}(\widehat X,v)\right)
=\mathsf T_{\mathcal G}\!\left(\mathcal D(\widehat X),v\right).
\label{eq:token_transition_preservation}
\end{equation}
\end{theorem}
\begin{proof}
Fix the entry origin $c$. Source-indexed byte simulation proves replay;
pointwise summary binding, worklist propagation, and suffix induction prove
local-effect execution. Both preserve $c$ through closure, normalization,
and finalization. Grouping reuses the same per-source relation. Public
exit closure restores the established representation invariants. Union
over sources gives the aggregate equation.
\end{proof}

\paragraph{Window and complete-output languages.}
For normalized $C_0$, allowed sets $\mathbf A=(A_0,\ldots,A_{H-1})$,
and $Y\in\{N,\mathrm{fac}\}$, define
\begin{equation}
\begin{aligned}
C_0^Y&=C_0,\\
C_{i+1}^Y&=\bigcup_{c\in C_i^Y}\ \bigcup_{v\in A_i}\mathsf T^Y(c,v),\\
E_i^Y&=\{(c,v,c'):c\in C_i^Y,\ v\in A_i,\ c'\in\mathsf T^Y(c,v)\}.
\end{aligned}
\label{eq:layered_frontier_definition}
\end{equation}
Here $\mathsf T^N=\mathsf T_{\mathcal G}$ and
$\mathsf T^{\mathrm{fac}}(c,v)
=\mathcal D_c(\widehat{\mathsf T}(\widehat X_i,v))$
for a representation of $C_i^{\mathrm{fac}}$. Stored edges use these
source--successor pairs, not Cartesian products of frontiers; graph nodes
retain their position.
Let $\mathcal L_{H,B_H}^Y$ contain the token sequences labeling paths from
$C_0$ to a configuration satisfying the common normalized boundary
predicate $B_H$.

\begin{theorem}[Normalized window-language preservation]
\label{thm:factorized_window_preservation}
Complete on-demand construction preserves $C_i^{\mathrm{fac}}=C_i^N$,
$E_i^{\mathrm{fac}}=E_i^N$, and
\begin{equation}
\boxed{
\mathcal L_{H,B_H}^{\mathrm{fac}}(C_0;\mathbf A)
=\mathcal L_{H,B_H}^{N}(C_0;\mathbf A).
}
\label{eq:window_language_preservation}
\end{equation}
\end{theorem}
\begin{proof}
Induct on position. Equal frontiers and
Eq.~\eqref{eq:sourcewise_token_transition_preservation} give the same
successors for each pair $(c,v)$, hence identical triples $(c,v,c')$ and
next frontiers. Output admissibility permits the next query. Effect groups
reconstruct these same edges when every admitted token label is retained.
Equal labeled paths and the same boundary predicate give equal window
languages.
\end{proof}

For a fixed token sequence, normalization commutation also gives
$X_H^N=N(X_H^{\mathrm{raw}})$. Thus raw-to-normalized boundary agreement
holds whenever $B_H^{\mathrm{raw}}(X)=\overline B_H(N(X))$.
This extra condition is unnecessary when comparing factorized execution
directly with the normalized reference.

Let $\mathcal L_{\mathrm{complete}}^Y$ contain finite token sequences
reaching complete-output acceptance under semantics $Y$.

\begin{theorem}[Complete-output language preservation]
\label{thm:normalization_complete_language}
$\mathcal L_{\mathrm{complete}}^N
=\mathcal L_{\mathrm{complete}}^{\mathrm{raw}}$.
Factorized acceptance agrees with both on completed queries.
\end{theorem}
\begin{proof}
Every complete accepting configuration satisfies $P$, so backward closure
ensures that every raw accepting path survives normalization. Normalization
only deletes alternatives, proving the reverse inclusion. The token theorem
then establishes factorized agreement. The same argument preserves the
existence of a finite legal completion, but not observations of dead threads.
\end{proof}

These statements preserve logical transitions and token languages, not
derivation multiplicities, arbitrary path weights, or resource-truncated
results.

\subsection{Weighted Inference under Token-Effect Grouping}
\label{sec:weighted_grouping}

Fix one denoising iteration and a finite layered window graph with initial
configuration $c_0$, horizon $H$, terminal predicate $B_H$, and fixed
nonnegative token weights $p_i(v)$. Forbidden tokens have zero weight.
All comparisons use the same normalized, sourcewise transition relation
$\mathsf T_{\mathcal G}$ and count each logical edge $(i,c,v,c')$ once.
Let $\Pi_q$ partition $\mathcal V$ and assume effect equivalence:
\begin{equation}
\mathsf T_{\mathcal G}(c,v)=\mathsf T_{\mathcal G}(c,G),
\qquad v\in G\in\Pi_{q(c)}.
\label{eq:weighted_effect_equivalence}
\end{equation}
Positions, endpoints, and group membership are retained. An accepting
path $\pi$ has weight $W(\pi)=\prod_{i=1}^H p_i(v_i)$; assume
$Z=\sum_{\pi\text{ accepting}}W(\pi)>0$. The following equalities use
exact arithmetic.

Define
\begin{equation}
M_i(G)=\sum_{v\in G}p_i(v),\qquad
m_i(G)=\max_{v\in G}p_i(v).
\label{eq:grouped_sum_max_weights}
\end{equation}
For $\oplus\in\{+,\max\}$, let $r_i^+=M_i$ and
$r_i^{\max}=m_i$. Disjoint group membership and
Eq.~\eqref{eq:weighted_effect_equivalence} give
\begin{equation}
\begin{aligned}
A_i^{\oplus}(c,c')
&:=\mathop{\bigoplus}_{v:\,c'\in\mathsf T_{\mathcal G}(c,v)}p_i(v)\\
&=\mathop{\bigoplus}_{G:\,c'\in\mathsf T_{\mathcal G}(c,G)}r_i^{\oplus}(G).
\end{aligned}
\label{eq:grouped_operator_equality}
\end{equation}
Both cases use multiplication to compose edge weights, with empty
reductions equal to zero. Thus grouping preserves the entire sum-product
or max-product transition operator, not only its support.

\begin{lemma}[Sum-product inference preservation]
\label{lem:grouped_sum_product}
Group-sum reduction with conditional token restoration preserves $Z$,
the full token-labeled path distribution, and all token marginals under
chain or tree inference using the conditionals below.
\end{lemma}
\begin{proof}
Equation~\eqref{eq:grouped_operator_equality} preserves every forward,
backward, and interval message. Set $\beta_H=B_H$ and
$\beta_{i-1}=A_i^+\beta_i$. Chain sampling selects a valid group edge and
then a token using
\begin{equation}
\Pr(G,c'\mid c)=\frac{M_i(G)\beta_i(c')}{\beta_{i-1}(c)},
\qquad
\Pr(v\mid G,c,c')=\frac{p_i(v)}{M_i(G)}.
\label{eq:grouped_chain_sampling}
\end{equation}
Their product is $p_i(v)\beta_i(c')/\beta_{i-1}(c)$, exactly the expanded
edge conditional. Multiplication along a path cancels the messages and
gives $W(\pi)/Z$.

For tree inference, write $K_{\ell:r}=A_{\ell+1}^+\cdots A_r^+$.
Sample the terminal state with probability
$K_{0:H}(c_0,b)B_H(b)/Z$ and an interval midpoint by
\begin{equation}
\Pr(C_k=z\mid C_\ell=a,C_r=b)
=\frac{K_{\ell:k}(a,z)K_{k:r}(z,b)}{K_{\ell:r}(a,b)}.
\label{eq:grouped_tree_midpoint}
\end{equation}
At leaf $i$ with endpoints $a,b$, choose $G$ with probability
$M_i(G)\mathbf1[b\in\mathsf T_{\mathcal G}(a,G)]/A_i^+(a,b)$,
then restore $v$ with probability $p_i(v)/M_i(G)$.
Associativity preserves the interval operators; cancellation of the
interval and leaf normalizers again gives $W(\pi)/Z$.
Only positive-mass events are sampled. Equality of full path distributions
implies equality of all token marginals.
\end{proof}

\begin{theorem}[Equivalence to Dang--Ermon sampling]
\label{thm:dang_grouping_equivalence}
Consider Algorithm~1 of \citet{dang2026constrained}, expressed in
the equivalent state-space formulation of their Appendix~A.2.
Use the same finite token-labeled graph, initial and terminal
factors, and logical edge-counting convention as above, with
per-token factors
\begin{equation}
g_i(c,v,c')
=
p_i(v)\mathbf1[c'\in\mathsf T_{\mathcal G}(c,v)].
\label{eq:dang_matched_token_factors}
\end{equation}
Under Lemma~\ref{lem:grouped_sum_product}, the grouped tree sampler
has the same partition function and output-token distribution
as this instance of Algorithm~1.
\end{theorem}

\begin{proof}
Their state-space transition matrix is
\begin{equation}
\mathsf M_i(c,c')
=
\sum_v g_i(c,v,c')
=
A_i^+(c,c').
\label{eq:dang_operator_identification}
\end{equation}
Equation~\eqref{eq:grouped_operator_equality} preserves this matrix
under group-sum reduction. Thus all bottom-up interval messages,
the root boundary distribution, and the conditional midpoint
distributions are unchanged.

At a leaf with endpoints $c,c'$, selecting a group and restoring
$v\in G$ gives
\begin{equation}
\frac{M_i(G)\mathbf1[c'\in\mathsf T_{\mathcal G}(c,G)]}
     {A_i^+(c,c')}
\frac{p_i(v)}{M_i(G)}
=
\frac{g_i(c,v,c')}{A_i^+(c,c')},
\label{eq:dang_leaf_equivalence}
\end{equation}
which is the same leaf token conditional. Only positive-mass
events are sampled. Induction over the sampling tree therefore
gives the same state-space joint distribution and hence the same
output-token distribution. The identical interval operators and
boundary factors also give the same partition function.
\end{proof}

\begin{lemma}[Max-product inference preservation]
\label{lem:grouped_max_product}
Group-max reduction preserves the Viterbi optimum under chain or tree
composition. All maximizing token sequences can be recovered when all
tied token and state-path choices are retained.
\end{lemma}
\begin{proof}
Equation~\eqref{eq:grouped_operator_equality} preserves every max-product
message. For a fixed grouped state path,
\begin{equation}
\max_{v_i\in G_i}\prod_{i=1}^H p_i(v_i)
=\prod_{i=1}^H m_i(G_i).
\label{eq:grouped_viterbi_restoration}
\end{equation}
Every token replacement is feasible by effect equivalence. Hence each
optimal grouped path lifts to an optimal token path by choosing group
maximizers, and every positive-weight expanded optimum is recovered this
way. Retaining all ties preserves the full maximizing sequence set.
A single argmax and backpointer return one optimum; identical selected
outputs additionally require the same tie-breaking rule.
\end{proof}

\paragraph{Parsing ambiguity and the probability measure.}
Let $n(y)$ be the number of accepting paths labeled by token sequence $y$
in the chosen token-level graph, and let $w(y)=\prod_i p_i(y_i)$.
Summing path weights gives
\begin{equation}
Q_{\mathrm{path}}(y)=\frac{n(y)w(y)}{\sum_z n(z)w(z)},
\qquad
Q_{\mathrm{seq}}(y)=
\frac{\mathbf1[n(y)>0]w(y)}{\sum_{z:n(z)>0}w(z)}.
\label{eq:path_sequence_measures}
\end{equation}
The second distribution counts each legal sequence once. Cancelling
positive $w(y)$ shows that the two distributions coincide if and only if
$n(y)$ is constant on the positive-weight legal support; unambiguity is
sufficient. A constant $n(y)=k$ gives the same normalized distribution
but scales its partition function by $k$. For example, equal-weight
sequences with path counts two and one have path probabilities $2/3$ and
$1/3$, rather than $1/2$ each. Thus language equivalence alone does not
establish sum-product equivalence across different parser representations.
Grouping preserves the chosen graph's measure, not an assumed absence
of parsing ambiguity.

These results require the stated weighted conditionals; greedy selection
by group mass is generally not equivalent. Extra token-dependent scores
require corresponding weighted reductions and token-restoration rules.
Distributional equality does not imply identical outputs under the same
random seed or floating-point reduction order.

\subsection{Complexity Analysis}
\label{sec:complexity_analysis}

We separate compilation and grammar propagation from model inference,
token-probability processing, and joint inference.

\paragraph{Context expansion.}
Let $Q_{\mathrm{loc}}$ be the local grammar states, $g>1$ the number of
possible continuation-frame symbols, $d_{\max}$ the stack-depth bound,
and $B$ the maximum number of remaining context assignments per
local-state--stack pair. Then
\begin{equation}
S_{d_{\max}}\le |Q_{\mathrm{loc}}|B\sum_{\ell=0}^{d_{\max}}g^\ell
=O(|Q_{\mathrm{loc}}|B g^{d_{\max}}).
\label{eq:bounded_context_expansion}
\end{equation}
The stack factor is $d_{\max}+1$ for $g=1$. The bound need not be attained.
For bounded-depth brackets with $g$ types, however, each opening-only
prefix has a distinct matching completion. A minimal complete DFA needs
one state per stack plus a sink:
\begin{equation}
S_{\mathrm{DFA}}^{\min}=1+\sum_{\ell=0}^{d_{\max}}g^\ell
=\Theta(g^{d_{\max}}).
\label{eq:bounded_nesting_lower_bound}
\end{equation}
This is a worst-case DFA representation gap, not a universal runtime bound.

\paragraph{Unordered properties.}
We prove Proposition~\ref{prop:unordered_properties}. Fix $k\ge1$ and distinct Unicode letters $a_1,\ldots,a_k$. Let $\mathcal G_k$ be the object schema with properties $a_1,\ldots,a_k$, all required, no additional properties, and the value schema \texttt{\{"const":0\}} for every property. Let $L_k$ be the set of its serializations, and write $P_i$ for the pair \texttt{"}$a_i$\texttt{":0}.

\emph{DFA bound.} For $S\subseteq\{1,\ldots,k\}$, let $u_S$ open the object and write the pairs $P_i$ with $i\in S$ in increasing order of $i$, separated by commas. When $\varnothing\ne S\ne\{1,\ldots,k\}$, $u_S$ also ends with a comma. Let $w_S$ write the remaining pairs in the same way and close the object, so that $u_Sw_S\in L_k$. Now take $S\ne T$ and assume without loss of generality that some $i\in S\setminus T$ exists. The string $u_Tw_S$ never writes $P_i$, because $i\notin T$ and $w_S$ writes only properties outside $S$. Since $a_i$ is required, $u_Tw_S\notin L_k$. The $2^k$ prefixes are therefore pairwise inequivalent, and by the Myhill--Nerode theorem~\citep{hopcroft2007automata} every DFA for $L_k$ has at least $2^k$ states.

\emph{Grammar bound.} Let $G$ be a context-free grammar for $L_k$, and let $s$ be its size, i.e., the total length of its productions. Let $R_k$ be the regular language of strings that open an object, write exactly $k$ pairs from $\{P_1,\ldots,P_k\}$ separated by commas, and close the object. A DFA for $R_k$ has $O(k)$ states. Every property is required and $R_k$ allows exactly $k$ pairs, so $L_k\cap R_k$ contains exactly the strings that write each $P_i$ once, in some order. The product construction for intersecting a context-free language with a regular language~\citep{hopcroft2007automata} gives a grammar for $L_k\cap R_k$ whose size is polynomial in $s$ and $k$. Let $h$ be the homomorphism that keeps each $a_i$ and erases every other character. Replacing every terminal of a grammar by its image under $h$ does not enlarge the grammar, and $h(L_k\cap R_k)$ is the set of all permutations of $a_1,\ldots,a_k$. Every context-free grammar for this set has size $2^{\Omega(k)}$~\citep{ellul2005regular}. Hence $\operatorname{poly}(s,k)\ge2^{\Omega(k)}$, and therefore $s=2^{\Omega(k)}$.

\emph{Factorized representation.} Under the parsing-environment factorization of Section~\ref{sec:pushdown_factorization}, the object needs local grammar states only for its delimiters, property names, separators, and values, so $O(k)$ of them suffice. The environment $\eta$ stores the set of properties already written, a $k$-bit value. Programs that complete a property name carry the guard ``$a_i\notin\eta$'' and the update ``add $a_i$ to $\eta$''. The program that closes the object carries the guard ``$\eta$ contains every required property''. All $2^k$ field histories share these programs.

\paragraph{Entry and instantiation counts.}
A dense context-expanded table has $S_{d_{\max}}|\mathcal V|$ entries;
factorized compilation has at most $|Q_{\mathrm{loc}}||\mathcal V|$
local-state--token entries. These are counts, not constant-time operations.
For reachable frontiers $\mathcal C_{i-1}$, define
\begin{equation}
\begin{aligned}
R_t&=\sum_{i=1}^H|\mathcal C_{i-1}|,\\
\Pi_{q,i}^t&=\{G\in\Pi_q:G\cap\mathcal V_i^t\ne\varnothing\},
\qquad K=\max_q|\Pi_q|,\\
N_{\mathrm{inst}}
&=\sum_{i=1}^H\sum_{c\in\mathcal C_{i-1}}|\Pi_{q(c),i}^t|
\le\min\{N_{\mathrm{query}},R_tK\},
\end{aligned}
\label{eq:effect_instantiation_complexity}
\end{equation}
where $N_{\mathrm{query}}=\sum_i\sum_{c\in\mathcal C_{i-1}}|\mathcal V_i^t|$.
Repeated configurations are counted per position. With full-vocabulary
candidates, the ungrouped count is $R_t|\mathcal V|$.

\paragraph{Construction cost.}
Let $T_{\mathrm{setup}}$ include shared schema lowering and static-distance
analysis. Let $C(q,v)$ include a cold entry's queue processing, improvement
propagation, and effect/key construction. Let $W_j$ include a top-level
instantiation's binding, all residual descendants, closures, suffix
compilation, and successor creation. For cold pairs
$\mathcal P_{\mathrm{cold}}\subseteq Q_{\mathrm{loc}}\times\mathcal V$,
actual top-level queries $|\mathcal J_{\mathrm{inst}}|\le N_{\mathrm{inst}}$,
and additional indexing and label/graph work $A_t$,
\begin{equation}
T_{\mathrm{construct}}=O\!\left(
T_{\mathrm{setup}}
+\sum_{(q,v)\in\mathcal P_{\mathrm{cold}}}C(q,v)
+\sum_{j\in\mathcal J_{\mathrm{inst}}}W_j+A_t
\right).
\label{eq:actual_construction_work}
\end{equation}
Each shared construction is charged once; recursive work included in
$W_j$ is not counted again. If $C(q,v)\le\tau_c$ and $W_j\le\tau_w$, then
\begin{equation}
T_{\mathrm{construct}}=O\!\left(
T_{\mathrm{setup}}+\tau_c|Q_{\mathrm{loc}}||\mathcal V|
+\tau_wR_tK+A_t\right).
\label{eq:combined_compilation_complexity}
\end{equation}
Neither per-entry cost is assumed constant for arbitrary grammars.
Grouping also does not remove the work needed to discover effects or
emit individual token labels.

\paragraph{Scope.}
For an $L$-byte entry, if every depth prefix is at least $-\kappa L$,
then $-\kappa L\le m\le0$ and each complete key admits at most
$\kappa L+1$ retained values, including its first admission.
This bounds re-expansion per key, not the number or size of keys.
Similarly, residual progress bounds a recursive chain by $L$ strict
decreases, not the total work of its branching tree. Local stacks and
pending programs may still grow, $R_t$ may approach $HS_{d_{\max}}$,
and $K$ may approach $|\mathcal V|$.
Cold assets are reusable across requests with the same grammar and binding
scope; graph reuse additionally requires compatible candidate sets and
boundary/protocol conditions. The analysis therefore characterizes avoided
replication without claiming an unconditional exponential decoding speedup.

\section{Additional Experimental Details}
\label{app:experimental_details}

This section provides additional implementation and evaluation settings for
the experiments in Section~\ref{sec:experiments}.

\subsection{Implementation and Decoding Settings}
\label{sec:detailed_settings}

Inference uses NVIDIA H20-3e GPUs with approximately 140\,GiB of visible
memory per device. Both models run in BF16 with model CUDA Graphs enabled.
LLaDA2.2~\citep{bie2026llada22} uses SGLang~\citep{zheng2024sglang} with
tensor parallelism across four GPUs, while
DiffusionGemma~\citep{diffusiongemma2026} uses
vLLM~\citep{kwon2023vllm} 0.28.0 with one GPU per model instance.
DConstrain implements cold and warm compilation in C++ and
joint inference in CUDA.

Cold-compilation experiments run on an Intel Xeon Platinum 8575C CPU. Each schema compilation is single-threaded, with up to 64 schemas compiled concurrently. Cold timing starts after tokenizer and service initialization and ends when the reusable compiled
artifact is ready for online use. Warm compilation and model inference are excluded from cold-compilation time.

\paragraph{LLaDA2.2.}
The adaptive runtime comparison uses 32-token physical blocks, with the first block's available positions determined by the prompt suffix. Generation is greedy with seed zero. MASK-to-token updates use a confidence
threshold of $0.5$, falling back to the highest-confidence MASK position when none qualifies. The token-to-token threshold is $0.0$; DELETE and SPLIT are disabled.  A block is committed after all MASK positions are filled and token
assignments stabilize. Generation terminates at an accepted EOS or the 2,048-token output limit, with bounded update/refinement rounds; budget exhaustion is counted as failure.

\paragraph{DiffusionGemma.}
Generation uses 256-token canvases, seed zero, and the native entropy-bound sampler. The entropy budget is $0.1$, and temperature decreases linearly from $0.8$ to $0.4$. Gumbel-max candidates are retained at entropy-selected positions while the remaining positions are re-noised, with native soft
self-conditioning preserved. Generation uses at most 48 denoising iterations per canvas and terminates
at EOS/EOT or the 2,048-token output limit.
\subsection{Compilation and Evaluation Settings}
\label{app:compilation_protocol}

The BFCL pool contains 2,495 cases, obtained by screening all \texttt{required} cases in BFCL and excluding cases whose schemas admit no valid instance. The JSONSchemaBench (JSB) pool contains all 5,754 training schemas, and the ComplexFuncBench (CFB) pool contains all 1,000 benchmark cases.

The main cold-compilation experiments and ablation study use the full corresponding pools. The main BFCL runtime tables report both all-attempted and common runtime-success cohorts, as defined below. Unless otherwise specified, other analyses and appendix experiments use fixed subsets of 300 BFCL cases and 300 JSB schemas, and the full 1,000-case CFB pool, as applicable. Table~\ref{tab:runtime_latency_matched} uses cases successfully completed by both DINGO-Opt and DConstrain within the designated pool. Failed compilation cases are not replaced.

For cold compilation, the nominal timeout is 180\,s per schema. Success rates include all attempted cases, while compilation latency statistics use each method's successfully compiled cases. Reported speedups divide the baseline's arithmetic mean latency by DConstrain's over paired successful cases. Timeouts and failures affect success rates but are excluded from these latency ratios.

\paragraph{BFCL completion and quality cohorts.}
In Table~\ref{tab:runtime_comparison}, cold-stage admission, runtime completion, strict-schema validity, and official BFCL semantic accuracy all use the full denominator of 2,495. $N$ counts runtime completions regardless of schema validity; TPS and TPF use those $N$ cases. Native-Opt has no constraint-compilation stage. DConstrain completes all 2,495 cases with 100\% cold admission, runtime completion, and strict-schema validity under both models and both order settings.

For Table~\ref{tab:runtime_common_success}, we define four independent cohorts, one for each combination of model and property-order condition, by intersecting the runtime-success sets of Native-Opt, LAVE-Opt, EPIC-Opt, and DConstrain. Schema-invalid completions are retained. The Schema-order and Any-order cohorts contain 1,874 and 2,201 cases for DiffusionGemma, and 2,448 and 2,149 cases for LLaDA2.2, respectively. DINGO-Opt completes all Schema-order cases; its validity, TPF, and semantic accuracy on the Schema-order common cohorts match DConstrain. Under Any-order, its cold-compilation success is 85.81\% for DiffusionGemma and 71.30\% for LLaDA2.2. Table~\ref{tab:runtime_comparison} reports its Any-order outcomes over the full pool. Including DINGO-Opt would substantially reduce the common-success cohorts in Table~\ref{tab:runtime_common_success}, so Table~\ref{tab:runtime_latency_matched} compares it with DConstrain on their paired successful cases. In Table~\ref{tab:runtime_comparison}, Schema-order DINGO-Opt has TPF 8.11 for DiffusionGemma and 6.86 for LLaDA2.2, semantic accuracy 37.19\% and 72.63\%, and TPS 23.13 and 101.32, respectively.

Schema validity is evaluated independently of solver acceptance and semantic correctness. The two main BFCL runtime tables report strict-schema validity and the official semantic evaluator separately; semantic accuracy is not filtered by schema validity. Other reported joint-correctness metrics additionally require schema validity. JSONSchemaBench validates outputs against the original schema. Timeouts and exhausted decoding budgets count as runtime failures; schema-invalid completions count against validity but remain completed runs.

\paragraph{Throughput accounting.}
For the main BFCL tables, TPF is the number of model non-stop token IDs divided by all generation forwards for the selected cases. CPU-generated completion text is excluded from model-token counts. The Overhead columns report inverse-TPS overhead over Native-Opt in the same model--order panel, $(\mathrm{TPS}_{\mathrm{Native\text{-}Opt}}/\mathrm{TPS}_{\mathrm{method}}-1)\times100\%$. In Table~\ref{tab:runtime_comparison} and the Schema-order panel of Table~\ref{tab:runtime_common_success}, DConstrain TPS subscripts give $\mathrm{TPS}_{\mathrm{DConstrain}}/\mathrm{TPS}_{\mathrm{DINGO\text{-}Opt}}$ in the same table and panel; its bold Overhead entries show the reduction factor $h_{\mathrm{DINGO}}/h_{\mathrm{DConstrain}}$, where $h=\mathrm{TPS}_{\mathrm{Native\text{-}Opt}}/\mathrm{TPS}-1$. Schema and Any denote schema-order and arbitrary-order settings; for Native-Opt, they identify the prompt condition.

\paragraph{Runtime timing.}
End-to-end latency is measured by wall-clock time at the request level, while GPU joint-inference stages are measured with CUDA events. Warm compilation measures the request-local successor-construction work, which overlaps with model execution; these intervals are therefore not additive. Cold compilation and artifact loading are excluded from runtime latency. Runtime latency statistics in Table~\ref{tab:runtime_latency_matched} are arithmetic means over the paired cases successfully completed by both compared methods.

\paragraph{Paired latency comparison.}
Table~\ref{tab:runtime_latency_matched} reports BFCL and JSB latency on the cases successfully completed by both DINGO-Opt and DConstrain in each dataset--order setting. E2E is the mean request latency. Solve/forward is the mean joint-inference latency per generation forward. On BFCL under Any-order, DConstrain reduces end-to-end latency by $6.65\times$ for LLaDA2.2 and $46.26\times$ for DiffusionGemma, and reduces solve latency per forward by $26.40\times$ and $108.54\times$, respectively.

\begin{table*}[t]
\centering
\setlength{\abovecaptionskip}{0pt}
\setlength{\belowcaptionskip}{8pt}
\caption{Paired runtime comparison on BFCL and JSONSchemaBench (JSB). Each setting uses the common successful cases of DINGO-Opt and DConstrain. Native-Opt is unconstrained decoding. E2E is the mean request latency; Solve/forward is the mean joint-inference latency per generation forward. Subscripts give DConstrain's speedup over DINGO-Opt.}
\label{tab:runtime_latency_matched}
\small
\setlength{\tabcolsep}{5.5pt}
\renewcommand{\arraystretch}{1.12}
\begin{adjustbox}{max width=\textwidth}
\begin{tabular}{@{}lccc@{\hspace{14pt}}ccc@{}}
\toprule
& \multicolumn{3}{c}{\textbf{LLaDA2.2}} & \multicolumn{3}{c}{\textbf{DiffusionGemma}} \\
\cmidrule(r){2-4}\cmidrule(l){5-7}
\textbf{Method} & Valid. & E2E & Solve/forward & Valid. & E2E & Solve/forward \\
& {\footnotesize\color{gray}(\%)\,$\uparrow$} & {\footnotesize\color{gray}(ms)\,$\downarrow$} & {\footnotesize\color{gray}(ms)\,$\downarrow$} & {\footnotesize\color{gray}(\%)\,$\uparrow$} & {\footnotesize\color{gray}(ms)\,$\downarrow$} & {\footnotesize\color{gray}(ms)\,$\downarrow$} \\
\midrule
\tgroup{7}{BFCL --- fixed order}
Native-Opt & 94.87 & 85.42 & \NA & 89.50 & 1948.23 & \NA \\
DINGO-Opt & \textbf{100.0} & 264.23 & 32.45 & \textbf{100.0} & 12874.32 & 216.72 \\
\ours \textbf{DConstrain} & \textbf{100.0} & \spd{115.21}{2.29} & \spd{3.16}{10.27} & \textbf{100.0} & \spd{2521.42}{5.11} & \spd{9.95}{21.78} \\
\tgroup{7}{BFCL --- any order}
DINGO-Opt & \textbf{100.0} & 813.20 & 74.18 & \textbf{100.0} & 145529.00 & 3195.39 \\
\ours \textbf{DConstrain} & \textbf{100.0} & \spd{122.35}{6.65} & \spd{2.81}{26.40} & \textbf{100.0} & \spd{3145.64}{46.26} & \spd{29.44}{108.54} \\
\tgroup{7}{JSB --- fixed order}
Native-Opt & 73.7 & 169.42 & \NA & 55.3 & 1894.82 & \NA \\
DINGO-Opt & \textbf{100.0} & 2282.87 & 172.57 & \textbf{100.0} & 54323.33 & 853.25 \\
\ours \textbf{DConstrain} & \textbf{100.0} & \spd{327.61}{6.97} & \spd{11.78}{14.65} & \textbf{100.0} & \spd{5388.21}{10.08} & \spd{34.17}{24.97} \\
\bottomrule
\end{tabular}
\end{adjustbox}
\end{table*}

\subsection{Engineering Optimizations}
\label{app:runtime_integration}

We complement DConstrain's compilation design with engineering optimizations that reduce execution overhead and graph-storage cost.

\paragraph{Native compilation pipeline.}
Cold and warm compilation execute in a native C++ pipeline, reducing interpreter overhead in grammar processing and transition instantiation. Reusable schema artifacts are stored in packed native arrays, avoiding per-edge Python objects and their allocation overhead.

\paragraph{Compact sparse graph storage.}
We store canvas-specific grammar graphs in compressed sparse row (CSR) format, using row offsets and contiguous edge arrays. This representation stores only existing transitions rather than allocating dense graph structures. CUDA joint-inference kernels consume the CSR representation directly, connecting grammar propagation with GPU execution.

DINGO-Opt adopts the same C++ pipeline and CSR representation while retaining finite-state compilation. Direct-PDA-Opt uses edge merging, shared token-membership storage, and grouped GPU reductions while interpreting reached parser configurations on demand, without compiling symbolic transition programs. Its + TG variant additionally applies conservative token grouping before interpretation (Appendix~\ref{app:direct_pda}).

\paragraph{Asynchronous execution.}
Native CPU workers run warm compilation during the first forward pass
of each canvas and grammar propagation during every forward pass.
Neither stage reads the model scores.
The scheduling interface supports both eager execution and model
CUDA Graph replay.
Data transfer and readiness tracking connect warm compilation to
GPU joint inference, which begins when both graph data and model scores
are available.
Score-dependent operations therefore remain on the critical path.
This overlap limits exposed warm-compilation latency, while workspace
reuse reduces repeated setup costs.
We measure end-to-end latency directly rather than summing overlapping
components, and report compilation and exposed costs in
Section~\ref{sec:compilation_results}.

\section{Additional Analyses of Constraint Semantics and Runtime}
\subsection{Warm Compilation Scaling and Runtime Overlap}
\label{app:warm_scaling}

DConstrain avoids materializing context-expanded token transitions during
cold compilation, but exact joint inference still requires the runtime
contexts that are reachable from the current decoding canvas.
Warm compilation instantiates context-specific successors on demand from shared transition programs.
This section examines whether the resulting deferred work becomes a
runtime bottleneck as parsing-environment and stack complexity increase,
and how this work is overlapped with model computation during execution.

\paragraph{Scaling with runtime-context complexity.}
Figure~\ref{fig:warm_scaling} complements the cold-compilation scaling
analysis in Figure~\ref{fig:ana} by measuring the corresponding
online cost.
We consider the two sources of runtime-context growth targeted by the
factorized representation.
For arbitrary-order objects, increasing the number of parameter fields
increases the number of possible parsing-environment configurations.
For nested structures, increasing nesting depth increases the diversity of
continuation-stack contexts that may be reached during grammar propagation.

On the CFB any-order sweep in Figure~\ref{fig:warm_scaling}(a), median CPU
warm-compilation time per 32-token block increases gradually with the
maximum number of parameter fields per object.
Even at the largest field counts, the median remains well below the
corresponding mean model-forward time of 9.09\,ms.
Thus, supporting arbitrary schema-valid property order introduces
additional online work, but the measured warm-compilation cost remains
small relative to the model computation with which it can overlap.

Figures~\ref{fig:warm_scaling}(b) and
\ref{fig:warm_scaling}(c) examine nesting depth on BFCL.
As depth increases from four to eight, warm-compilation work increases,
consistent with a larger set of reachable stack contexts.
The increase is visible both in the total warm-compilation work accumulated
within a request and in the mean cost of an individual warm-compilation
call.
At depth eight, the mean per-call CPU cost remains below the approximately
9\,ms mean model-forward time, while end-to-end request latency remains
comparatively stable across the depth sweep.
Across the depth sweep, mean warm-compilation cost stays below the mean model-forward time, and end-to-end latency remains comparatively stable.

\begin{figure*}[t]
\centering
\includegraphics[width=\textwidth]{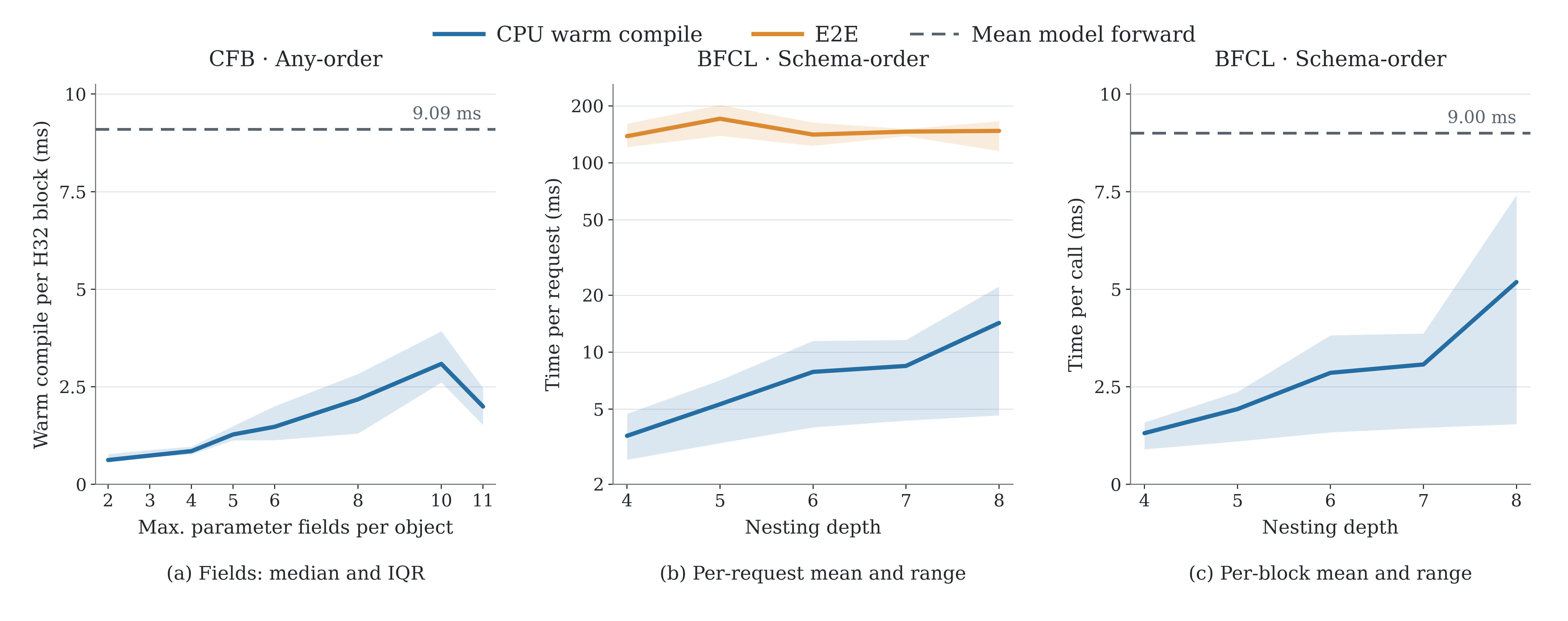}
\caption{
Warm-compilation scaling with runtime-context complexity.
(a) CPU warm-compilation time per 32-token block on CFB with arbitrary
property order, grouped by the maximum number of parameter fields per
object; the line and band report the median and interquartile range.
(b) Per-request end-to-end latency and accumulated CPU warm-compilation
time on BFCL as nesting depth increases.
(c) Mean CPU warm-compilation time per call over the same depth sweep.
Bands in (b,c) show the observed range across requests.
Dashed lines denote the corresponding mean model-forward time.
}
\label{fig:warm_scaling}
\end{figure*}

\paragraph{Warm compilation during asynchronous execution.}
The scaling experiments above characterize how online work changes with
specific sources of context complexity.
Figure~\ref{fig:warm_compile_overlap} in the main text reports warm-compilation timings from actual asynchronous execution across the broader evaluation workloads. Native CPU workers run score-independent warm compilation concurrently with GPU model forward computation. Joint inference starts when both the graph and model scores are ready, so only warm compilation remaining at that point is exposed on the decoding critical path. The figure sorts the measured CPU warm-compilation times and shows the corresponding mean model-forward duration as a reference; the percentage annotations summarize the fraction of warm-compilation calls below that duration.

For LLaDA2.2 on CFB, arbitrary property order increases warm-compilation
cost relative to schema order, as expected from the additional
parsing-environment alternatives, while the distribution remains below
the corresponding mean forward time over the plotted samples.
On JSONSchemaBench, which contains substantially more heterogeneous
schemas, the distribution has a longer tail: 97.17\% of warm-compilation
calls remain below the 8.90\,ms mean model-forward duration, while the
largest observed call reaches 34.61\,ms.

The same pattern holds with the substantially larger 256-token canvases of
DiffusionGemma.
On BFCL and CFB with arbitrary property order, 99.47\% and 94.20\% of
warm-compilation calls, respectively, are below the corresponding mean
model-forward durations of 26.4 and 26.7\,ms.
On JSONSchemaBench, 99.3\% are below the 25.64\,ms mean forward duration,
with a maximum observed warm-compilation time of 30.32\,ms.

\paragraph{Implications for deferred instantiation.}
Together, Figures~\ref{fig:warm_scaling} and~\ref{fig:warm_compile_overlap} clarify
the trade-off introduced by two-stage compilation.
DConstrain shares compiled transition logic across the runtime contexts reached during grammar propagation.
It performs
context-specific instantiation only when the corresponding configurations
become reachable.
The resulting online work grows with runtime-context complexity, as
expected, but remains modest relative to model-forward computation over the
evaluated regimes and is overlapped with model execution by the asynchronous schedule. The request-level means in Table~\ref{tab:runtime_latency_matched} account for this overlap directly.

\section{Additional Experimental Results}
\label{app:additional_results}

These experiments complement the main runtime and compilation results in Section~\ref{sec:experiments}. We examine direct pushdown interpretation, tree-parallelized automata inference, compiled representations, and compilation coverage.

\subsection{Comparison with Direct Pushdown Interpretation}
\label{app:direct_pda}

The ablations in Section~\ref{sec:ablation} vary reuse within DConstrain. Here we additionally evaluate a baseline that computes $(c,v)\mapsto\mathsf T_{\mathcal G}(c,v)$ by directly interpreting the pushdown parser at reached runtime configurations, without materializing a context-expanded DFA or compiling symbolic transition programs. This comparison tests whether on-demand interpretation alone can provide efficient successor construction.

\paragraph{Baseline variants.}
\emph{Direct-PDA-Opt} already merges output edges with common destinations and uses shared token-membership storage and grouped GPU reductions. \emph{Direct-PDA-Opt + TG} additionally groups tokens \emph{before interpretation}: bytes are equivalent only when their raw transitions agree at every local RTN node, and token-trie paths with identical byte-class sequences are merged. It also reuses unions of token memberships. Calls, returns, and environment updates still execute in the interpreter, and all original token identities and model scores are retained. This conservative grouping uses no symbolic transition programs and differs from DConstrain's local token-effect grouping.

\paragraph{Evaluation setting.}
We evaluate all variants on the same fixed set of 300 BFCL cases, with LLaDA2.2, BF16, tensor parallelism of four, batch size one, and 32-token blocks. All variants use the same name-first tool-call format, arbitrary parameter-field order, a nesting-depth limit of four, prompt, and MAP decoding. These controlled runs execute graph construction and model forward serially, with a fresh runtime per request. E2E spans the first prefill through runtime and parser cleanup after generation; cold preparation is excluded. Figure~\ref{fig:direct_pda_baseline} uses the original paired runs. Figure~\ref{fig:direct_pda_grouping} reruns DConstrain and the grouped baseline in the same service on the same 300 cases. All 300 cases produce identical output tokens across variants and satisfy their schemas.

\paragraph{Runtime comparison.}
Pre-interpretation grouping substantially reduces the direct interpreter's online cost, but a large gap remains. In the grouped paired run, mean E2E is 878.10\,ms for Direct-PDA-Opt + TG and 136.63\,ms for DConstrain, a $6.43\times$ speedup. The warm-compilation curves show that repeated runtime interpretation remains expensive even after conservative grouping. Together with the compiler ablations, these results support compiling and reusing transition programs as a contribution beyond simply avoiding DFA expansion.

\begin{figure*}[t]
\centering
\includegraphics[width=\textwidth]{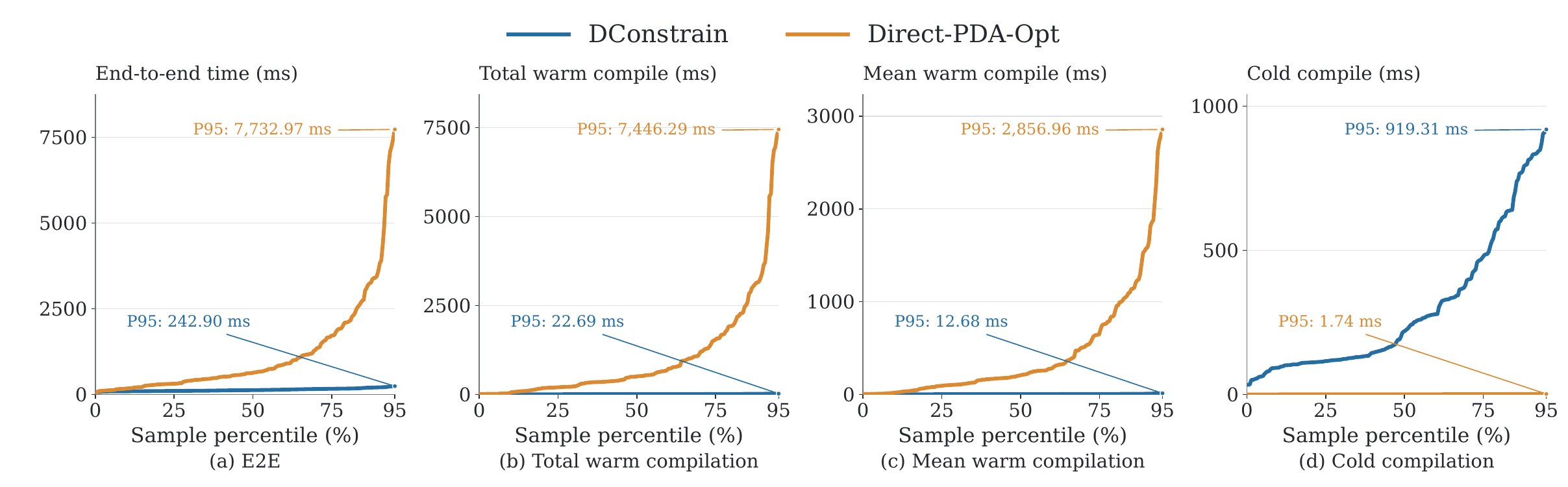}
\caption{DConstrain versus Direct-PDA-Opt without pre-interpretation token grouping on 300 BFCL cases with LLaDA2.2. Panels report (a) complete-request E2E, (b) accumulated warm-compilation time per request, (c) each request's mean time per warm-compilation call, and (d) cold-compilation time. Each curve shows its latency percentiles from P0 to P95; the upper 5\% is omitted without rescaling the percentile axis.}
\label{fig:direct_pda_baseline}
\end{figure*}

\begin{figure*}[t]
\centering
\includegraphics[width=\textwidth]{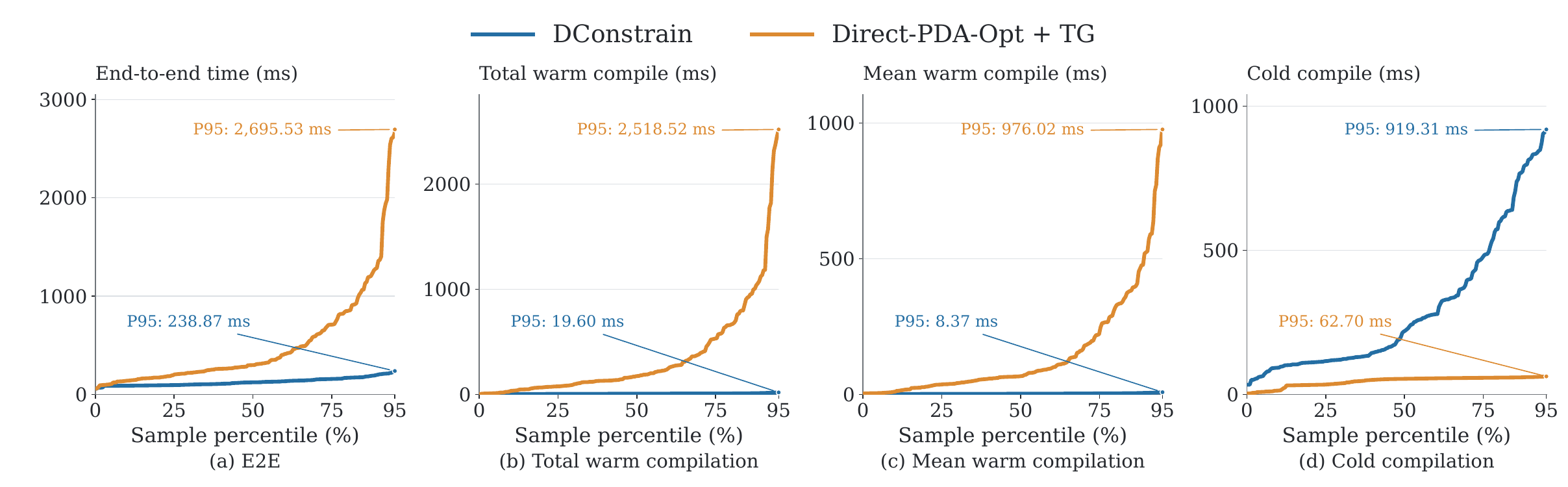}
\caption{DConstrain versus Direct-PDA-Opt + TG, which adds conservative pre-interpretation token grouping and membership reuse. Both methods are rerun on the same 300 BFCL cases. Panel definitions and the P0--P95 display follow Figure~\ref{fig:direct_pda_baseline}.}
\label{fig:direct_pda_grouping}
\end{figure*}
\subsection{Tree-Parallelized Automata Inference}
\label{app:tree_sampling}

DConstrain changes how the canvas-specific grammar graph is constructed while preserving the token-labeled successor relation consumed by downstream automata-based inference. It is therefore solver-agnostic rather than specialized to the Viterbi solver used in our main experiments. We verify this compatibility by applying both Viterbi inference and the inference algorithm of \citet{dang2026constrained} directly to DConstrain graphs, using both sequential and tree-parallelized implementations.

We further examine whether tree parallelism can reduce the solver cost underlying Table~\ref{tab:runtime_latency_matched}. All four solver variants are evaluated on the same 300 BFCL cases using both DINGO and DConstrain graph representations. Table~\ref{tab:tree_sampling} reports results for two cases near the 25th and 95th percentiles of DINGO graph node count, denoted P25 and P95. The selected cases and inference objectives are matched across graph representations.

We consider the log-depth tree construction proposed by Dang and Ermon~\citep{dang2026constrained} and apply the same reduction to MAP inference for a direct comparison with the Viterbi-style decoding used by DINGO~\citep{suresh2025dingo}. Let $L$ denote the decoding-window length, $S$ the automaton states, and $E$ the grammar-transition structure. Conventional Viterbi inference performs position-wise dynamic programming, where state scores at position $i$ depend on those at position $i-1$, yielding a sequential dependency across the decoding window. Tree-based inference instead composes adjacent transition operators in a binary reduction tree, reducing the sequential depth from $O(L)$ to $O(\log L)$. For MAP inference, the same construction can be performed over the max-product semiring~\citep{goodman1999semiring}, preserving the Viterbi objective~\citep{viterbi1967} while exposing tree-level parallelism.

This reduction trades sequential depth for substantially more computation and intermediate state. Under the state-space formulation, sequential message passing requires
\begin{equation}
O\!\left(
L\left(|S|^2 + |E||\mathcal V|\right)
\right)
\end{equation}
work, whereas the tree construction requires
\begin{equation}
O\!\left(
L\left(|S|^3 + |E||\mathcal V|\right)
\right)
\end{equation}
work while reducing the critical path to $O(\log L)$. Dang and Ermon characterize this trade-off as an $|S|$-fold increase in message-passing FLOPs in exchange for greater GPU parallelism. The memory cost can also increase substantially: each interval in the reduction tree materializes a state-to-state transition summary of size $O(|S|^2)$. Retaining the summaries required by the subsequent top-down pass therefore requires $O(L|S|^2)$ intermediate storage, compared with $O(L|S|)$ backpointer storage for conventional Viterbi reconstruction, or $O(|S|)$ score storage when reconstruction is unnecessary. This quadratic dependence on the automaton state space can become expensive for large structured-output graphs.

\begin{table*}[t]
\centering
\setlength{\abovecaptionskip}{0pt}
\setlength{\belowcaptionskip}{10pt}
\caption{Solver latency (ms) and peak solver GPU memory (GiB). All four solver variants are evaluated on 300 BFCL cases; the table displays cases near the 25th and 95th percentiles of DINGO graph node count. P25/P95 refer to graph size, not latency quantiles. ``Ours'' denotes DConstrain; ``seq.'' and ``tree'' denote sequential and tree-parallelized execution. DANG denotes the inference method of \citet{dang2026constrained}.}
\label{tab:tree_sampling}
\small
\setlength{\tabcolsep}{5pt}
\renewcommand{\arraystretch}{1.15}
\begin{tabular}{@{}llN{3.2}N{1.2}@{\hspace{14pt}}N{5.2}N{2.2}@{}}
\toprule
& & \multicolumn{2}{c}{P25 BFCL schema}
& \multicolumn{2}{c}{P95 BFCL schema} \\
\cmidrule(lr){3-4}\cmidrule(l){5-6}
Graph & Method
& {Time (ms) $\downarrow$} & {Mem. (GiB) $\downarrow$}
& {Time (ms) $\downarrow$} & {Mem. (GiB) $\downarrow$} \\
\midrule
\multirow{4}{*}{Ours}
& Viterbi seq. & \bfseries 3.900 & \bfseries 0.010 & \bfseries 30.850 & \bfseries 0.031 \\
& DANG seq. & 4.700 & 0.010 & 48.970 & 0.085 \\
& Viterbi tree & 30.800 & 1.140 & 7824.192 & 12.552 \\
& DANG tree & 60.700 & 1.560 & 4344.171 & 12.299 \\
\midrule
\multirow{4}{*}{DINGO}
& Viterbi seq. & 39.800 & 0.010 & 425.081 & 0.077 \\
& Viterbi tree & 7.130 & 0.210 & 880.384 & 12.838 \\
& DANG seq. & 331.320 & 0.020 & 13536.261 & 0.085 \\
& DANG tree & 11.400 & 0.290 & 3577.134 & 17.625 \\
\bottomrule
\end{tabular}
\end{table*}

Table~\ref{tab:tree_sampling} shows that the benefit of tree parallelism depends strongly on graph size. On the P25 DINGO graph, tree-parallelized Viterbi reduces latency from $39.8\,\mathrm{ms}$ to $7.13\,\mathrm{ms}$, a $5.58\times$ speedup, but remains slower than sequential Viterbi on the DConstrain graph at $3.90\,\mathrm{ms}$. On the P95 DINGO graph, tree-parallelized Viterbi instead increases latency from $\num{425.081}\,\mathrm{ms}$ to $\num{880.384}\,\mathrm{ms}$, while sequential Viterbi on the DConstrain graph requires only $\num{30.850}\,\mathrm{ms}$. Peak solver GPU memory on the DINGO graph increases from $\num{0.010}\,\mathrm{GiB}$ to $\num{0.210}\,\mathrm{GiB}$ at P25 and from $\num{0.077}\,\mathrm{GiB}$ to $\num{12.838}\,\mathrm{GiB}$ at P95, approximately $21\times$ and $167\times$, respectively. Tree parallelism also increases both latency and memory for Viterbi on the DConstrain graph in these two cases. These measurements illustrate that reducing sequential depth does not guarantee lower solver cost: as the graph grows, the additional dense intermediate computation and storage can dominate the available parallelism.

\subsection{Compiled Representations with DiffusionGemma}
\label{app:gemma_representation}

Table~\ref{tab:compiled_representation_gemma} complements the LLaDA2.2 results in Table~\ref{tab:compiled_representation_llada}. DConstrain's median recorded memory footprint is 4.10\,MB on BFCL, 4.23\,MB on JSONSchemaBench, and 4.57\,MB on ComplexFuncBench; the corresponding P95 values are 5.39\,MB, 28.41\,MB, and 5.09\,MB. State counts use the three-method successful intersection, while edge and memory quantiles use each method's successful cases.

\begin{table*}[t]
\centering
\setlength{\abovecaptionskip}{0pt}
\setlength{\belowcaptionskip}{10pt}
\caption{Compiled representations with the DiffusionGemma tokenizer
on fixed sets of 300 BFCL cases, 300 JSONSchemaBench schemas, and 1,000 ComplexFuncBench cases. Edges are reported in thousands; edge and memory columns show P50 and P95.
Mean state counts use the three-method successful intersection
(258, 202, and 421 cases, respectively);
edge and memory quantiles use each method's successful cases.
States denote automaton states for baselines and local RTN nodes
for DConstrain; edges are token-level versus symbolic.
$^{\ddagger}$Memory measures baseline edge-array payloads versus
DConstrain's recorded native-snapshot payloads, in decimal MB,
rather than total process memory.}
\label{tab:compiled_representation_gemma}

\small
\setlength{\tabcolsep}{4pt}
\renewcommand{\arraystretch}{1.15}

\begin{tabular}{@{}llN{3.2}N{4.2}N{4.2}N{3.2}N{3.2}@{}}
\toprule
\multicolumn{7}{@{}l}{
\textbf{Compiled representation and memory footprint}
} \\
\addlinespace[4pt]
Dataset & Method
& {States $\downarrow$}
& \multicolumn{2}{c}{Edges ($10^3$) $\downarrow$}
& \multicolumn{2}{c}{Memory (MB)$^{\ddagger}$ $\downarrow$} \\
\cmidrule(lr){4-5}
\cmidrule(l){6-7}
& & {Mean} & {P50} & {P95} & {P50} & {P95} \\
\midrule

BFCL
& DINGO & 235.21 & 884.925 & 2070.611 & 42.49 & 99.46 \\
& NFA & 240.21 & 593.875 & 1774.719 & 28.53 & 85.25 \\
\ours & \textbf{DConstrain} & \bfseries 208.61 & \bfseries 0.763 & \bfseries 4.003 & \bfseries 4.10 & \bfseries 5.39 \\

\midrule

JSONSchemaBench
& DINGO & 187.28 & 591.323 & 2066.609 & 28.40 & 99.24 \\
& NFA & 223.56 & 596.349 & 3001.778 & 28.65 & 144.16 \\
\ours & \textbf{DConstrain} & \bfseries 169.92 & \bfseries 0.851 & \bfseries 51.803 & \bfseries 4.23 & \bfseries 28.41 \\

\midrule

ComplexFuncBench
& DINGO & 460.21 & 1776.397 & 3252.160 & 85.32 & 156.17 \\
& NFA & 465.21 & 2365.893 & 3842.167 & 113.62 & 184.51 \\
\ours & \textbf{DConstrain} & \bfseries 424.23 & \bfseries 2.354 & \bfseries 3.703 & \bfseries 4.57 & \bfseries 5.09 \\

\bottomrule
\end{tabular}
\end{table*}

To complement the any-order compilation results in Table~\ref{tab:cold_compilation}, Table~\ref{tab:native_anyorder_representation} compares the compiled representations produced by DINGO-Opt and DConstrain under the native tool-call protocols. All representation statistics are computed on the same successful intersection within each setting. Because DConstrain succeeds on every evaluated case, the intersection size $n$ is determined by the finite-state baseline's compilation coverage. The results show that arbitrary property order leads to very large context-expanded transition structures for DINGO-Opt, whereas DConstrain retains compact local RTN nodes and symbolic transitions.

\begin{table*}[t]
\centering
\setlength{\abovecaptionskip}{0pt}
\setlength{\belowcaptionskip}{10pt}
\caption{Compiled representations under native tool-call protocols with arbitrary parameter-field order. Statistics use the common successful intersection within each setting. States denote DFA states for DINGO-Opt and local RTN nodes for DConstrain; edges denote token-level and symbolic transitions, respectively. DConstrain successfully compiles every evaluated case.}
\label{tab:native_anyorder_representation}
\small
\setlength{\tabcolsep}{4pt}
\renewcommand{\arraystretch}{1.15}
\begin{adjustbox}{max width=\textwidth}
\begin{tabular}{@{}llclr rrrrr@{}}
\toprule
Tokenizer & Dataset & Depth & Success & Method & States & \multicolumn{2}{c}{Edges ($10^3$)} & \multicolumn{2}{c}{Payload (MB)} \\
\cmidrule(lr){7-8}\cmidrule(l){9-10}
& & & & & Mean & P50 & P95 & P50 & P95 \\
\midrule
\multirow{4}{*}{LLaDA2.2}
& BFCL & 4 & 213/300 & DINGO-Opt & \num{655.92} & \num{49124.943} & \num{161307.594} & \num{2358.16} & \num{7743.37} \\
& & & & \textbf{DConstrain} & \num{164.98} & \num{3.272} & \num{8.872} & -- & -- \\
\addlinespace[3pt]
& CFB & 4 & 99/1000 & DINGO-Opt & \num{409.19} & \num{77163.564} & \num{169703.331} & \num{3704.11} & \num{8146.34} \\
& & & & \textbf{DConstrain} & \num{247.52} & \num{9.648} & \num{15.140} & -- & -- \\
\midrule
\multirow{4}{*}{DiffusionGemma}
& BFCL & 4 & 2141/2495 & DINGO-Opt & \num{1737.94} & \num{9178.192} & \num{65604.918} & \num{440.59} & \num{3149.94} \\
& & & & \textbf{DConstrain} & \num{481.19} & \num{0.957} & \num{3.584} & \bfseries \num{3.46} & \bfseries \num{4.20} \\
\addlinespace[3pt]
& CFB & 4 & 137/1000 & DINGO-Opt & \num{1379.69} & \num{19669.069} & \num{62942.461} & \num{944.23} & \num{3021.61} \\
& & & & \textbf{DConstrain} & \num{693.85} & \num{2.278} & \num{3.614} & \bfseries \num{3.66} & \bfseries \num{3.86} \\
\bottomrule
\end{tabular}
\end{adjustbox}
\vspace{5pt}
\end{table*}

\subsection{Cold Compilation with Additional Tokenizers}
\label{app:additional_tokenizers}

To evaluate whether DConstrain's compilation behavior depends on the two tokenizers used in the main experiments, we additionally compile the same structured-output workloads with the Dream 7B and LLaDA 8B tokenizers. These experiments evaluate compilation only and do not run model inference. All settings use arbitrary property order and the same recursive structured-output grammar configuration.

Table~\ref{tab:anyorder_other_tokenizers} shows that DConstrain retains high compilation coverage and sub-second median cold-compilation latency across both additional tokenizers. BFCL and ComplexFuncBench compile successfully in all cases, while JSONSchemaBench succeeds on 281 of 300 evaluated schemas. The identical local-state counts across tokenizers reflect the tokenizer-independent pushdown grammar structure, while token-dependent symbolic transitions, payload size, and compilation latency vary with tokenization. These results indicate that the factorized compilation design is not specific to the LLaDA2.2 or DiffusionGemma tokenizers used in the main experiments.

\begin{table*}[t]
\centering
\setlength{\abovecaptionskip}{0pt}
\setlength{\belowcaptionskip}{10pt}
\caption{Any-order cold compilation with Dream 7B and LLaDA 8B tokenizers. All runs use the same recursive structured-output grammar configuration and nesting bounds of 8/8. Statistics use successful compilations; Success reports coverage. States are mean local RTN node counts, edges are symbolic transitions, and payload is the schema-exclusive native-snapshot footprint in decimal MB. Cold compilation uses 64 concurrent single-threaded workers.}
\label{tab:anyorder_other_tokenizers}
\small
\setlength{\tabcolsep}{4pt}
\renewcommand{\arraystretch}{1.15}
\begin{tabular}{@{}llcrrrrrrr@{}}
\toprule
Tokenizer & Dataset & Success & States & \multicolumn{2}{c}{Edges ($10^3$)} & \multicolumn{2}{c}{Payload (MB)} & \multicolumn{2}{c}{Cold (s)} \\
\cmidrule(lr){5-6}\cmidrule(lr){7-8}\cmidrule(l){9-10}
& & & Mean & P50 & P95 & P50 & P95 & P50 & P95 \\
\midrule
Dream 7B & BFCL & 300/300 & \num{416.25} & \num{0.916} & \num{4.236} & \num{2.77} & \num{4.08} & \num{0.360} & \num{1.052} \\
& JSONSchemaBench & 281/300 & \num{616.65} & \num{1.025} & \num{157.527} & \num{2.78} & \num{24.71} & \num{0.384} & \num{12.852} \\
& ComplexFuncBench & 1000/1000 & \num{770.05} & \num{2.829} & \num{4.387} & \num{3.22} & \num{3.69} & \num{0.672} & \num{0.967} \\
\midrule
LLaDA 8B & BFCL & 300/300 & \num{416.25} & \num{0.853} & \num{4.003} & \num{2.09} & \num{3.03} & \num{0.309} & \num{0.814} \\
& JSONSchemaBench & 281/300 & \num{616.65} & \num{0.929} & \num{138.532} & \num{2.08} & \num{14.90} & \num{0.328} & \num{10.780} \\
& ComplexFuncBench & 1000/1000 & \num{770.05} & \num{2.630} & \num{4.063} & \num{2.45} & \num{2.80} & \num{0.596} & \num{0.838} \\
\bottomrule
\end{tabular}
\end{table*}

\subsection{Failure Analysis on JSONSchemaBench}
\label{app:jsb_failure_analysis}

We further analyze the remaining DConstrain compilation failures on JSONSchemaBench with the LLaDA2.2 tokenizer. Table~\ref{tab:jsb_failure_analysis} categorizes all 1,072 failed schemas out of 5,754, corresponding to an 18.6\% failure rate and an 81.4\% success rate. Most failures arise from unsupported schema semantics rather than compiler execution failures: semantic-coverage limitations account for 789 cases (13.7\% of the full benchmark and 73.6\% of all failures). Resource limits account for another 274 cases (4.8\% of the benchmark), primarily RSS exhaustion and the 180\,s compilation timeout. Only nine cases (0.2\%) are attributable to known implementation errors or unresolved compiler behavior.

\begin{table}[t]
\centering
\setlength{\abovecaptionskip}{0pt}
\setlength{\belowcaptionskip}{10pt}
\caption{Failure breakdown for DConstrain on JSONSchemaBench with the LLaDA2.2 tokenizer. Percentages are relative to all 5,754 schemas and rounded to one decimal place; the total is calculated from counts before rounding.}
\label{tab:jsb_failure_analysis}
\small
\setlength{\tabcolsep}{5pt}
\renewcommand{\arraystretch}{1.10}
\begin{tabular}{@{}lrr@{}}
\toprule
Failure cause & Cases & Full set (\%) \\
\midrule
\multicolumn{3}{@{}l}{\textbf{Schema coverage and semantic limitations}} \\
Unsupported keyword or keyword combination & 285 & 5.0 \\
\texttt{oneOf} branch exclusivity not provable & 178 & 3.1 \\
General \texttt{allOf} intersection unsupported & 127 & 2.2 \\
Nested resource scopes (\texttt{id}/\texttt{\$id}) unsupported & 127 & 2.2 \\
Nested schema-dialect switching unsupported & 28 & 0.5 \\
Dynamic-property constraints conflict with empty-object subset & 38 & 0.7 \\
Regular expression outside supported dialect & 4 & 0.1 \\
Tuple-style \texttt{items} unsupported & 1 & 0.0 \\
\texttt{\$ref} and outer-type intersection not provable & 1 & 0.0 \\
\addlinespace[3pt]
\multicolumn{3}{@{}l}{\textbf{Resource and construction limits}} \\
State/transition/closure or pattern-size budget exceeded & 68 & 1.2 \\
RSS limit exceeded & 134 & 2.3 \\
Compilation timeout & 72 & 1.3 \\
\addlinespace[3pt]
\multicolumn{3}{@{}l}{\textbf{Implementation and unresolved failures}} \\
\texttt{\_deduplicate\_json\_values} implementation error & 4 & 0.1 \\
Duplicate \texttt{enum} values rejected & 3 & 0.1 \\
Compiled language unexpectedly empty & 2 & 0.0 \\
\midrule
\textbf{Total failures} & \textbf{1,072} & \textbf{18.6} \\
\bottomrule
\end{tabular}
\end{table}

The dominant failure modes therefore reflect the current supported schema fragment rather than failures of the factorized compilation mechanism itself. In particular, \texttt{oneOf}/\texttt{allOf} reasoning and nested resource semantics account for a substantial fraction of unsupported cases. Resource failures form the second-largest category, showing that although DConstrain substantially reduces context expansion, sufficiently complex schemas can still exceed finite construction or memory budgets.

\paragraph{Localizing frontend coverage limitations.}
To separate frontend coverage from factorized execution, we replayed the
178 \texttt{oneOf} branch-disjointness rejections and 127 general
\texttt{allOf} rejections in Table~\ref{tab:jsb_failure_analysis} using the
archived frontend from the reported experiment. All 305 rejections were
reproduced by the schema-to-parser frontend alone, without invoking the
factorized token compiler. These failures therefore precede factorized
execution; their localization does not establish that downstream
compilation would succeed after extending the frontend.

\paragraph{Frontend acceptance versus semantic preservation.}
We examined the same 305 original schemas using XGrammar~\citep{dong2025xgrammar}.
With \texttt{strict\_mode=True}, XGrammar 0.2.8 returned grammars for
298 schemas. A bounded search identified 124 distinct schemas for which
the native matcher accepted a complete JSON text that violated the
original schema. We independently replayed all 124 counterexamples and
confirmed their invalidity with \texttt{jsonschema} 4.26.0; 104 were also
confirmed by \texttt{fastjsonschema} 2.22.2. We additionally found valid
instances for 86 schemas, establishing their satisfiability. Replaying
the same invalid texts in XGrammar 0.1.27 reproduced violations for
123 schemas. Table~\ref{tab:jsb_semantic_audit} summarizes these results.

The audit used an empty tokenizer and the native complete-string matching
interface, requiring an accepting terminal state. It therefore tests
membership in the compiled character language, rather than compilation
against the LLaDA2.2 vocabulary or the likelihood of generating these
texts during model inference. Validation followed each schema's declared
dialect, defaulting to Draft 2020-12 when unspecified. Optional
\texttt{format} assertions and network reference retrieval were disabled.

\begin{table}[t]
\centering
\setlength{\abovecaptionskip}{0pt}
\setlength{\belowcaptionskip}{10pt}
\caption{Targeted semantic audit of 305 historical DConstrain rejections.
All counts refer to schemas. Counterexamples are complete JSON texts
accepted by XGrammar but rejected under the original schema. Rows use
DConstrain's first rejection cause, which need not match the constraint
violated by the witness. Counterexample counts are lower bounds from a
bounded search; absence of a witness does not establish semantic preservation.}
\label{tab:jsb_semantic_audit}
\small
\setlength{\tabcolsep}{4pt}
\renewcommand{\arraystretch}{1.10}
\begin{tabular}{@{}lrrrr@{}}
\toprule
Historical rejection
& Cases
& \shortstack{Compiled\\XG 0.2.8}
& \shortstack{Counterexamples\\XG 0.2.8}
& \shortstack{Same witnesses\\XG 0.1.27} \\
\midrule
\texttt{oneOf} disjointness & 178 & 175 & 43 & 42 \\
General \texttt{allOf}     & 127 & 123 & 81 & 81 \\
\midrule
\textbf{Total} & \textbf{305} & \textbf{298} & \textbf{124} & \textbf{123} \\
\bottomrule
\end{tabular}
\end{table}

For example, JSB schema \texttt{o41599} specifies \texttt{oneOf} over
integer, string, number, boolean, and null. Both tested XGrammar versions
accept \texttt{0}, which satisfies both the integer and number branches
and thus violates the exactly-one requirement. Similarly, schema
\texttt{o44455} explicitly requires an object and contains a multi-branch
\texttt{allOf}, yet both matchers accept \texttt{0}. Both examples were
confirmed by both validators.

\paragraph{Scope of the conclusions.}
Frontend acceptance alone is insufficient evidence that the original
schema constraints are preserved. This targeted audit does not estimate
overall semantic accuracy, establish DConstrain's semantic correctness
on all accepted schemas, or imply that every rejected schema is
intrinsically difficult. We distinguish execution correctness relative
to supported reference-parser semantics from frontend coverage and
resource-bounded construction. Incomplete coverage, resource failures,
and implementation issues remain limitations of the current system.

\paragraph{Limitations.}
DConstrain substantially reduces the cost of nested constraints during cold compilation by retaining a symbolic RTN representation rather than expanding stack contexts into finite states, and this representation can in principle express context-free structure. However, factorization does not eliminate the worst-case complexity of joint inference. During warm compilation, uncertainty over unresolved token choices can propagate across the decoding canvas and induce a rapidly growing set of reachable runtime contexts; in the worst case, this expansion can still be exponential in nesting depth or other sources of branching. DConstrain therefore extends the practical nesting depth substantially relative to finite-state baselines, but it does not provide unbounded or uniformly efficient compilation for arbitrary deeply nested grammars.

\end{document}